\documentclass[10pt,twocolumn,letterpaper]{article}

\usepackage{cvpr}
\usepackage{times}
\usepackage{epsfig}
\usepackage{graphicx}
\usepackage{amsmath}
\usepackage{amssymb}

\usepackage{amsthm,paralist}
\usepackage{graphicx}
\usepackage{relsize,comment}
\usepackage{multirow}
\usepackage{wrapfig}
\usepackage{subcaption}
\usepackage{algorithm}
\usepackage{algorithmic}
\usepackage{booktabs}
\usepackage{amsmath}
\usepackage{amssymb}
\usepackage{epsfig}
\usepackage{bbm}
\usepackage{amsmath,amssymb,bm}
\usepackage{relsize}
\usepackage{graphicx}
\usepackage[inline]{enumitem}
\usepackage{array}
\usepackage{colortbl, color}
\usepackage{soul}
\usepackage[export]{adjustbox}

\usepackage[svgnames]{xcolor}
\usepackage{empheq}
\usepackage{colortbl, color}

\colorlet{shadecolor}{FloralWhite!60}
\colorlet{titleshade}{OliveDrab!15}
\newsavebox{\mysaveboxM} 
\newsavebox{\mysaveboxT} 
\newcommand*\Garybox[2][Example]{%
	\sbox{\mysaveboxM}{#2}%
	\sbox{\mysaveboxT}{\fcolorbox{black}{titleshade}{\enspace#1\enspace}}%
	\sbox{\mysaveboxM}{%
		\parbox[b][\ht\mysaveboxM+.5\ht\mysaveboxT+.5\dp\mysaveboxT][b]{%
			\wd\mysaveboxM}{#2}%
	}%
	\sbox{\mysaveboxM}{%
		\fcolorbox{black}{shadecolor}{%
			\makebox[\linewidth-3.5em]{\usebox{\mysaveboxM}}%
		}%
	}%
	\usebox{\mysaveboxM}%
	\makebox[0pt][r]{%
		\makebox[\wd\mysaveboxM][l]{%
			\raisebox{\ht\mysaveboxM-0.5\ht\mysaveboxT
				+0.5\dp\mysaveboxT-0.5\fboxrule}{\usebox{\mysaveboxT}}%
		}%
	}%
}%

\colorlet{shadecoloralg}{White}
\colorlet{titleshadealg}{FloralWhite!60}
\newcommand*\GaryboxAlg[2][Example]{%
	\sbox{\mysaveboxM}{#2}%
	\sbox{\mysaveboxT}{\fcolorbox{black}{titleshadealg}{\enspace#1\enspace}}%
	\sbox{\mysaveboxM}{%
		\parbox[b][\ht\mysaveboxM+.5\ht\mysaveboxT+.5\dp\mysaveboxT][b]{%
			\wd\mysaveboxM}{#2}%
	}%
	\sbox{\mysaveboxM}{%
		\fcolorbox{black}{shadecoloralg}{%
			\makebox[\linewidth-3.5em]{\usebox{\mysaveboxM}}%
		}%
	}%
	\usebox{\mysaveboxM}%
	\makebox[0pt][r]{%
		\makebox[\wd\mysaveboxM][c]{%
			\raisebox{\ht\mysaveboxM-0.5\ht\mysaveboxT
				+0.5\dp\mysaveboxT-0.5\fboxrule}{\usebox{\mysaveboxT}}%
		}%
	}%
}%

\newtheorem{theorem}{Theorem}

\newtheorem{lemma}[theorem]{Lemma}

\theoremstyle{definition}

\newtheorem{remark}[theorem]{Remark}
\newtheorem{fact}[theorem]{Fact}
\newtheorem{obsn}[theorem]{Observation}

\newcommand{\R}{\mathbb{R}}

\newcolumntype{P}[1]{>{\centering\arraybackslash}p{#1}}
\newcolumntype{g}{>{\columncolor[rgb]{0.89, 0.89, 1}}c}

\newcommand\mtiny[1]{\mbox{\scriptsize\ensuremath{#1}}}

\def\argmax{\mathrm{argmax}}
\def\argmin{\mathrm{argmin}}

\def\argmax{\mathrm{argmax}}
\def\argmin{\mathrm{argmin}}

\definecolor{mygrey}{HTML}{929292}
\definecolor{mylightgrey}{HTML}{E7E7E7}

\definecolor{mypink}{HTML}{99195E}
\definecolor{mylightpink}{HTML}{EEDBE7}

\definecolor{mygreen}{HTML}{017100}
\definecolor{mylightgreen}{HTML}{DCE9DA}

\newcommand*{\figuretitle}[1]{%
	\hskip 2pt \centering
	{
		\tiny #1
		\vskip -2pt}
}

\usepackage[breaklinks=true,bookmarks=false]{hyperref}

\cvprfinalcopy 

\def\cvprPaperID{****} 

\begin{document}

\title{FairALM: Augmented Lagrangian Method for Training Fair Models with \\ Little Regret}

\author{{\normalsize \textbf{Vishnu Suresh Lokhande \textsuperscript{\rm 1} }}\\
{\tt\small lokhande@cs.wisc.edu}
\and
{\normalsize \textbf{Aditya Kumar Akash \textsuperscript{\rm 1}}}\\
{\tt\small aakash@wisc.edu}
\and
{\normalsize \textbf{Sathya N. Ravi \textsuperscript{\rm 2}}} \\
{\tt\small sathya@uic.edu}
\and
{\normalsize \textbf{Vikas Singh \textsuperscript{\rm 1}}} \\
{\tt\small vsingh@biostat.wisc.edu}
\and
{\textsuperscript{\rm 1}University of Wisconsin-Madison}\\ 
{\textsuperscript{\rm 2}University of Illinois at Chicago}
}

\maketitle

\begin{abstract}
Algorithmic decision making based on computer vision and machine learning 
technologies continue to permeate our lives.
But issues related to biases of these models 
and the extent to which they treat certain segments of the population
unfairly, have led to concern in the general public. 
It is now accepted that because of biases in the datasets 
we present to the models, a fairness-oblivious training will 
lead to unfair models. An interesting topic is the study of mechanisms via which
the {\em de novo} design or training of the model can be informed by fairness measures.
Here, we study mechanisms that impose fairness concurrently while training the model.  
While existing fairness based approaches in vision have largely relied 
on training adversarial modules together with the primary
classification/regression task, in an effort 
to remove the influence of the protected attribute or variable, 
we show how ideas based on well-known optimization concepts can provide a  
simpler alternative. In our proposed scheme, imposing fairness just
requires specifying the protected attribute and utilizing our optimization routine. We provide a 
detailed technical analysis and present experiments demonstrating 
that various fairness measures from the literature can be reliably imposed 
on a number of training tasks in vision in a manner that is interpretable. A project page is available on \href{https://github.com/lokhande-vishnu/FairALM}{//GitHub}.
\end{abstract}

\section{Introduction}
Fairness and non-discrimination is a core tenet of modern society. Driven
by advances in vision and machine learning systems, algorithmic decision making continues to
permeate our lives in important ways. Consequently, ensuring that the decisions taken
by an algorithm do not exhibit serious biases is no longer a hypothetical topic, rather
a key concern that has started informing legislation \cite{Goodman_Flaxman_2017} (e.g., Algorithmic Accountability act).
On one extreme, some types of biases can be bothersome -- a biometric access system
could be more error-prone for faces of persons from certain skin tones \cite{buolamwini2018gender} or a search for
{\tt \small homemaker} or {\tt \small programmer} may
return gender-stereotyped images \cite{bolukbasi2016man}. 
But there are serious ramifications as well -- an individual may get pulled aside for an intrusive check while
traveling \cite{zuber2014critical} or a model may decide to pass on an individual for a job interview after digesting his/her social media content\cite{chin_2019,heilweil_2019}. 
Biases in automated systems in estimating recidivism within the criminal judiciary have been reported \cite{ustun2016learning}.
There is a growing realization that these problems need to be
identified and diagnosed, and then promptly addressed. In the worst case, if no solutions
are forthcoming, we must step back and reconsider the trade-off between
the benefits versus the harm of deploying such systems, on a case by case basis. \\
\begin{figure}[!t]
	\centering
	\includegraphics[width=0.9\columnwidth]{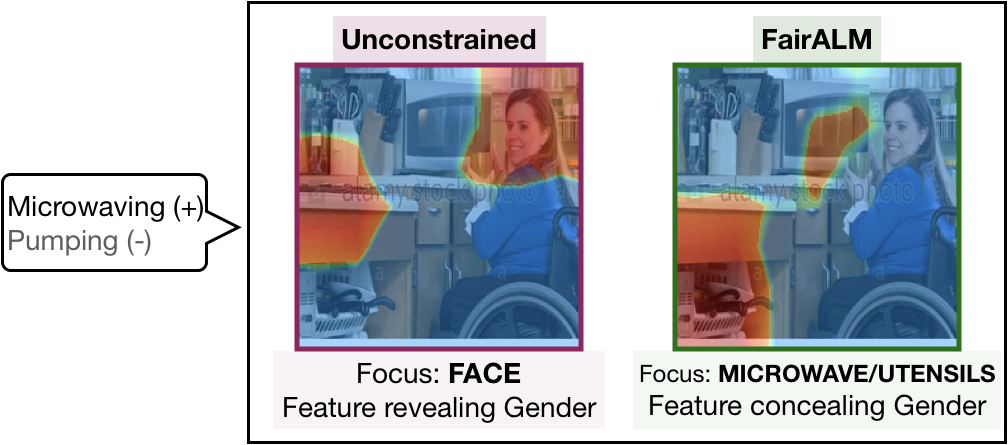}
	\caption{\label{fig:intro_fig} \footnotesize {The heat maps of an unconstrained model and a fair model are depicted in this figure. The models are trained to predict the target label \textit{Microwaving} ( indicated by a $(+)$). The fair model attempts to make unbiased predictions with respect to sensitive attribute \textit{gender}. In this example, it is observed that the heat maps of an unconstrained model are concentrated around gender revealing attributes such as the face of person. Alternatively, the heat maps of the fair model are concentrated around non-gender revealing attributes, such as utensils and microwave, which also happen to be more aligned to the target label.}}
	\vskip -0.2in
\end{figure}

{\bf What leads to unfair learning models?} One finds that learning methods in general tend to amplify biases that exist in the training dataset \cite{zhao2019gender}.
While this creates an incentive for the organization training the model to curate datasets that are ``balanced''
in some sense, from a practical standpoint, it is often difficult to collect data that is balanced along
multiple predictor variables that are ``protected'', e.g., gender, race and age.
If a protected feature is correlated with the response variable, a learning model can {\em cheat} and
find representations from other features that are collinear or a good surrogate for the protected variable. 
A thrust in current research is devoted to devising ways to mitigate such shortcuts.  
If one
does not have access to the underlying algorithm, a recent result \cite{hardt2016equality} shows
the feasibility of finding thresholds that can impose certain fairness criteria. Such a threshold
search can be post-hoc applied to any learned model. But in various cases, 
because of the characteristics of the dataset, 
a fairness-oblivious training will lead to biased models. An
interesting topic is the study of mechanisms via which
the {\em de novo} design or training of the model can be informed by fairness measures.

{\bf Some general strategies for Fair Learning.}
Motivated by the foregoing issues, recent work which
may broadly fall under the topic of {\em algorithmic fairness}
has suggested several concepts or measures
of fairness that can be incorporated within the learning model. 
While we will discuss the details shortly, these include 
demographic parity \cite{yao2017beyond}, equal odds and equal opportunities \cite{hardt2016equality}, and
disparate treatment \cite{zafar2017fairness}. In general,
existing work can be categorized into a few distinct categories.
The {\em first} category of methods attempts to modify the representations of the data to ensure
fairness. While different methods approach this question in different ways,
  the general workflow involves imposing fairness {\em before} a subsequent use of 
standard machine learning methods \cite{calmon2017optimized,kamiran2010classification}.
The {\em second} group of methods adjusts the decision boundary of an already trained classifier towards
making it fair as a {\em post}-processing step while trying to incur as little deterioration
in overall performance as possible \cite{goh2016satisfying,fish2016confidence,woodworth2017learning}.
While this procedure is convenient and fast,
it is not always guaranteed to lead to a fair model
without sacrificing accuracy. Part of the reason
is that the search space for a fair solution in the post-hoc tuning
is limited. Of course, we may impose fairness during training directly as
adopted in the {\em third} category of papers such as
\cite{zafar2017parity,bechavod2017penalizing}, and the
approach we take here.
Indeed, if we are training the model from scratch and have knowledge of
the protected variables, there is
little reason not to incorporate this information directly {\em during} model training. 
In principle, this strategy provides the maximum control over the model.  
From the formulation standpoint, it is 
slightly more involved because it requires satisfying a fairness constraint derived from 
one or more fairness measure(s) in the literature, while concurrently learning the model parameters.
The difficulty varies depending both on the primary task (shallow versus deep model)
as well as the specific fairness criteria. For instance, if one were using a deep network for classification,
we would need to devise ways to enforce constraints on the {\em output} of the network, efficiently.

{\bf Scope of this paper and contributions.}
Many studies on fairness in learning and vision are somewhat recent and
were partly motivated in response to more than a few controversial reports in the
news media. As a result,
the literature on mathematically sound and practically sensible fairness measures that can 
still be incorporated while training a model is still in a nascent stage. In vision, current
approaches have largely relied on training adversarial modules in conjunction with the primary
classification or regression task, to remove the influence of the protected attribute.
In contrast, the {\bf contribution} of our work is to provide a simpler alternative.
We show that a number of fairness measures in the literature can be incorporated by viewing them as
constraints on the {\em output} of the learning model.
This view allows adapting ideas from constrained optimization, to devise ways in which
training can be efficiently performed in a way that at termination, the model parameters correspond 
to a fair model. For a practitioner, this means that no changes in the
architecture or model are needed: imposing fairness only requires specifying the protected attribute,
and utilizing our proposed optimization routine.

\section{A Primer on Fairness Functions}
In this section, we introduce basic notations and briefly review several fairness measures described in the literature.\\
{\bf Basic notations.} \label{sec:nots} We denote classifiers using $h:x\mapsto y$ where $x$ and $y$ are random variables that
represent the features and labels respectively.
A {\em protected} attribute is a random variable $s$ on the same probability space as $x$ and $y$ --
for example, $s$ may be gender, age, or race. Collectively, a training example would be $z:= (x, y, s)$. So, our goal is to learn $h$ (predict $y$ given $x$) while
{\em imposing fairness-type constraints} over $s$. We will use $\mathcal{H} = \{h_1, h_2, \hdots, h_N\}$ to denote a
set/family of possible classifiers and $\Delta^N$ to denote the probability simplex in $\R^N$, i.e.,
$\Delta:=\{q:\sum_{i=1}^Nq_i=1,q_i\geq 0 \}$ where $q_i$ is the $i$-th coordinate of $q$. 

Throughout the paper, we will assume that the distribution of $s$ has  finite support. Unless explicitly specified, we will assume that $y\in\{0,1\}$ in the main paper. For each $h\in\mathcal{H}$, we will use $e_h$ to denote the misclassification rate of $h$ and $e_{\mathcal{H}}\in\R^N$ to be the vector containing all misclassification rates. We will use superscript to denote conditional expectations. That is, if $\mu_h$ corresponds to expectation of some function $\mu$ (that depends on $h\in\mathcal{H}$), then the conditional expectation/moment of $\mu_h$ with respect to $s$ will be denoted by $\mu_{h}^s$. With a slight abuse of notation, we will use $\mu_h^{s_0}$ to denote the elementary conditional expectation $\mu_h|(s=s_0)$ whenever it is clear from the context. We will use $d_{h}$ to denote the
{\em difference} between the conditional expectation of the two groups of $s$, that is, $d_h := \mu_h^{s_0} - \mu_h^{s_1}$.  For example, let $s$ be the random variable representing 
gender, that is, $s_0$ and $s_1$ may correspond to male and female. Then, $e_h^{s_i}$ corresponds to the misclassification rate of $h$ on group $s_i$, and $d_h=e_h^{s_0} - e_h^{s_1}$. Finally, $\mu_h^{s_i,t_j}:=\mu_h|(s=s_i,t=t_j)$ denotes the elementary conditional expectation with respect to two random variables $s,t$. 

\subsection{Fairness through the lens of Confusion Matrix}

Recall that a {\em fairness} constraint corresponds to a performance requirement of a classifier $h$ on
subgroups of features $x$ {\em induced} by a protected attribute $s$.
For instance, say that $h$ predicts the credit-worthiness $y$ of an individual $x$.
Then, we may require that $e_h$ be ``approximately'' the same across individuals for different races given by $s$. Does it follow that functions/metrics that are used to evaluate fairness may be written in terms of the
error of a classifier $e_h$ {\em conditioned} on the protected variable $s$ (or in other words $e_h^s$)? Indeed, it does turn out to be the case! In fact, many widely used functions in practice can be viewed as imposing constraints on the confusion matrix as our intuition suggests. We will now discuss few common fairness metrics to illustrate this idea. 

{\bf (a) Demographic Parity (DP) \cite{yao2017beyond}.} A classifier $h$ is said to satisfy Demographic Parity (DP) if $h(x)$ is {\em independent}
of the protected attribute $s$. Equivalently, $h$ satisfies DP if  $d_h=0$ where we set $\mu_{h}^{s_i} = e_h^{s_i}$ (using notations introduced above). DP can be seen as equating the total false positives and false negatives between the confusion matrices of the two groups. We denote DDP by the difference of the demographic parity between the two groups. 

{\bf (b) Equality of Opportunity (EO) \cite{hardt2016equality}.} \label{sec:EO} A classifier $h$ is said to satisfy EO
if $h(x)$ is independent of the protected attribute $s$ for $y\in\{0,1\}$. Equivalently, $h$ satisfies EO if $d_h^{y}=0$ where we set $\mu_h^{s_i} = e_h^{s_i}|(y\in\{0,1\})=: e_h^{s_i,y_j}$ conditioning on both $s$ and $y$.
 Depending on the choice of $y$ in $\mu_h^{s_i}$, we get two different metrics: \begin{enumerate*}[label=(\roman*)]
	\item {\bf $y=0$} corresponds to $h$ with equal {\it  False Positive Rate (FPR)} across $s_i$ \cite{chouldechova2017fair}, whereas 
	\item {\bf $y=1$} corresponds to $h$ with equal {\it  False Negative Rate (FNR)} across $s_i$ \cite{chouldechova2017fair}. 
	\end{enumerate*}
Moreover, $h$ satisfies {\it Equality of Odds} if $d_h^{0}+d_h^{1}=0$, i.e., $h$ equalizes both TPR and FPR across $s$ \cite{hardt2016equality}. 
We denote the difference in EO by DEO.

{\bf (c) Predictive Parity (PP) \cite{celis2019classification}.} A classifier $h$ satisfies PP if the likelihood of making a misclassification among the positive predictions of the classifier is independent of the protected variable $s$. 
Equivalently, $h$ satisfies PP if $d_h^{\hat{y}}=0$ where we set $\mu_{h_i}^{s_i} = e_h^{s_i}|(\hat{y}=1)$. It corresponds to matching the False Discovery Rate between the confusion matrices of the two groups.

\section{How to learn fair models?} At a high level, the optimization problem that we seek to solve is written as,\begin{align}
 \min_{h\in\mathcal{H}} \mathbb{E}_{z:(x,y,s)\sim \mathcal{D}} \mathcal{L}(h;(x,y)) ~~\text{subject to}~~ h \in \mathcal{F}_{d_h},	\label{eq:fairopt}	
\end{align}
where $\mathcal{L}$ denotes the loss function that measures the accuracy of $h$ in predicting $y$ from $x$, and $\mathcal{F}_{d_h}$ denotes the set of {\em fair} classifiers. Our approach
to solve \eqref{eq:fairopt} {\em provably efficiently} involves two main steps: \begin{enumerate*}[label=(\roman*)]\item first, we reformulate  problem \eqref{eq:fairopt} to compute a posterior distribution $q$ over $\mathcal{H}$; \item second, we incorporate fairness as {\em soft} constraints on the output of $q$ using the augmented Lagrangian of Problem \eqref{eq:fairopt}.
  We assume that we have access to sufficient number of samples to approximate $\mathcal{D}$ and solve the empirical version of Problem \eqref{eq:fairopt}. 
\end{enumerate*}

\subsection{From Fair Classifiers to Fair Posteriors} \label{sec:model}
The starting point of our development is based on the following simple result that follows directly from the definitions of fairness metrics in Section \ref{sec:nots}:
\begin{obsn}\label{obs1:rand}
Fairness metrics such as DP/EO are {linear} functions of $h$, whereas PP takes a linear {fractional} form due to the conditioning on $\hat{y}$, see \cite{celis2019classification}.
\end{obsn}
Observation \ref{obs1:rand} immediately implies that $\mathcal{F}_{d_h}$ can be represented using linear (fractional) equations in $h$. To simplify the discussion,
we will focus on the case when $\mathcal{F}_{d_h}$ is given by the DP metric. Hence, we can reformulate  \eqref{eq:fairopt} as,
\begin{align}
\label{eq:alm_twoplayer}
\min_{q\in\Delta} \quad \sum_i q_i e_{h_i} \text{ s.t. } q_i (\mu_{h_i}^{s_0} - \mu_{h_i}^{s_1}) = 0 \quad \forall i \in [N],
\end{align}
where $q$ represents a distribution over $\mathcal{H}$. 
\subsection{Imposing Fairness via Soft Constraints} In general, there are two ways of treating the $N$ constraints $q_id_{h_i}=0$ in Problem \eqref{eq:alm_twoplayer} viz., \begin{enumerate*}[label=(\roman*)]\item  as {\em hard constraints}; or \item as {\em soft constraints}.
\end{enumerate*} Algorithms that can handle explicit constraints efficiently require access to an efficient oracle that can minimize a linear or quadratic function over the feasible set in {\em each} iteration. Consequently, algorithms that incorporate hard constraints come with high per-iteration computational cost since the number of constraints is (at least) linear in $N$, and is not applicable in large scale settings. Hence, we propose to use algorithms that incorporate fairness as soft constraints. With these two minor modifications, we will now describe our approach to solve problem \eqref{eq:alm_twoplayer}.

\section{Fair Posterior from Proximal Dual}
Following the reductions approach in \cite{agarwal2018reductions}, we first write the Lagrangian dual problem of DP constrained risk minimization problem \eqref{eq:alm_twoplayer} using dual variables $\lambda$ as,
\begin{align}
\label{eq: randomized_almminmax}
\max_{\lambda\in\R^N}\min_{q\in\Delta} L(q, \lambda):= \langle q,  e_h \rangle  + \lambda \langle q, \mu_h^{s_0} - \mu_h^{s_1} \rangle 
\end{align}

{\bf Interpreting the Lagrangian.} Problem \ref{eq: randomized_almminmax} can be understood as a game between two players a $q$-player and a $\lambda$-player \cite{cotter2018optimization}. We recall an important fact regarding the dual problem \eqref{eq: randomized_almminmax}:\begin{fact}\label{rmk:dualns} The objective function of the dual problem \eqref{eq: randomized_almminmax}  is {\em always nonsmooth} with respect to  $\lambda$ because of the inner minimization problem in $q$.
\end{fact}  Technically, there are two main reasons why optimizing nonsmooth functions can be challenging \cite{duchi2012randomized}: \begin{enumerate*}[label=(\roman*)]\item finding a descent direction in high dimensions $N$ can be challenging; and \item subgradient methods can be slow to converge in practice.\end{enumerate*} Due to these difficulties arising from Fact \ref{rmk:dualns}, using a first order algorithm such as gradient descent to solve the dual problem in \eqref{eq: randomized_almminmax} directly can be problematic, and may be  suboptimal. 

{\bf Accelerated optimization using Dual Proximal Functions.} To overcome the difficulties due to the nonsmoothness of the dual problem, we propose to {\em augment} the Lagrangian with a proximal term. 
Specifically, for some $\lambda_T$, the augmented Lagrangian function can be written as,
\begin{align}
\resizebox{0.9\hsize}{!}{$L_T(q, \lambda) = \langle q, e_h \rangle  + \lambda \langle q, \mu_h^{s_0} - \mu_h^{s_1} \rangle - \frac{1}{2\eta} (\lambda - \lambda_T)^2$}\label{eq:alprox}
\end{align}
Note that, as per our simplified notation, $L_T \equiv L_{\lambda_T}$. The following lemma relates the standard Lagrangian in \eqref{eq: randomized_almminmax} with its proximal counterpart in \eqref{eq:alprox}. 
\begin{lemma}
	\label{lemma: alm_vs_lm}
	At the optimal solution $(q^*,\lambda^*)$ to $L$, we have $\max_{\lambda} \min_{q\in\Delta} L = \max_{\lambda}\min_{q\in\Delta} L_{\lambda^*}$.
\end{lemma}
This is a standard property of proximal objective functions, where $\lambda^*$ forms a fixed point of $\min_{q\in \Delta} L_{\lambda^*}(q, \lambda^*)$ (section 2.3 of \cite{parikh2014proximal}). Intuitively, Lemma \ref{lemma: alm_vs_lm} states that $L$ and $L_T$ are not at all different for optimization purposes.
\begin{remark}While the augmented Lagrangian $L_T$ still may be nonsmooth, the proximal (quadratic) term can be exploited to design {\em provably} faster optimization algorithms as we will see shortly. 
\end{remark}

\begin{algorithm}[t]
	\caption{FairALM: Linear Classifier}
	\label{alg:fair_alm_linear}
	\begin{algorithmic}[1]
		\STATE {\em Notations:} Dual step size $\eta$ \\ \quad $h_t \in \{h_1, h_2, \hdots, h_N \}$. 
		\STATE {\em Input:} Error Vector $e_\mathcal{H}$,\\ \quad Conditional mean vector $\mu_\mathcal{H}^{s}$\\
		\STATE {\em Initializations: $\lambda_0 = 0$}
		\FOR{$t = 0, 1, 2, ...,T$}
		\STATE (Primal) $h_t \leftarrow \argmin_i (e_{h_i} + \lambda_t (\mu_{h_i}^{s_0} - \mu_{h_i}^{s_1}))$ \label{alg: ht_update}
		\STATE (Dual) $\lambda_{t+1} \leftarrow \lambda_t + \eta (\mu_{h_t}^{s_0} - \mu_{h_t}^{s_1}) / t$ \label{alg: lambda_t update}
		\ENDFOR
		\STATE {\em Output:} $h_T$
	\end{algorithmic}
\end{algorithm}

\section{Our Algorithm -- FairALM}
{It is  common \cite{agarwal2018reductions,cotter2018optimization,kearns2017preventing} to consider the minimax problem in \eqref{eq:alprox} as a zero sum game between the $\lambda$-player and the $q$-player. The Lagrangian(s) $L_T$ (or $L$) specify the cost which the $q$-player pays to the $\lambda$-player after the latter makes its choice. An iterative procedure leads to a regret minimizing strategy for the $\lambda$-player \cite{shalev2012online} and a best response strategy for the $q$-player \cite{agarwal2018reductions}. While the $q$-player's move relies on the availability of an efficient \textit{oracle} to solve the minimization problem, $L_T(q, \lambda)$, being a linear program in $q$ makes it less challenging. We describe our algorithm in Alg.~\ref{alg:fair_alm_linear} and call it \textit{FairALM: Linear Classifier}.


\subsection{Convergence Analysis}
As the game with respect to $\lambda$ is a maximization problem, we get a reverse regret bound as shown in the following Lemma. Due to space, proofs appear in the Appendix.
\begin{lemma}
	\label{lemma: regretbound}
	Let $r_t$ denote the reward at each round of the game. The reward function $f_t(\lambda)$ is defined as $f_t(\lambda) = \lambda r_t - \frac{1}{2\eta} (\lambda - \lambda_t)^2$.
        We choose $\lambda$ in round $T+1$ to maximize the cumulative reward: $\lambda_{T+1} = \argmax_{\lambda} \sum_{t=1}^T f_t(\lambda)$.  Define $L = \max_t |r_t|$.   The following bound on the cumulative reward holds, for any $\lambda$
	\begin{align}
	\resizebox{0.9\hsize}{!}{$\sum_{t=1}^T \bigg(\lambda r_t - \frac{1}{2\eta} (\lambda - \lambda_t)^2 \bigg) \le \sum_{t=1}^T \lambda_t r_t + \frac{\eta}{2}L^2 \mathcal{O}(\log T)$}
	\end{align}	
\end{lemma}
The above lemma indicates that the cumulative reward grows in time as $\mathcal{O}(\log T)$. The proximal term in the augmented Lagrangian gives us a {\em better} bound than an $\ell_2$ or an entropic regularizer (which provides a $\sqrt{T}$ bound \cite{shalev2012online}).

Next, we evaluate the cost function $L_T(q, \lambda)$ after $T$ rounds of the game. We observe that the average play of both the players converges to a saddle point with respect to $L_T(q, \lambda)$. We formalize this in the following theorem,
\begin{theorem}
	Recall that $d_h$ represents the difference of conditional means. Assume that $|| d_h||_{\infty} \le L$ and consider $T$ rounds of the game described above. Let the average plays of the $q$-player be $\bar q = \frac{1}{T}\sum_{t=1}^T q_t$ and the $\lambda$-player be $\bar \lambda = \frac{1}{T}\sum_{t=1}^T \lambda_t$. Then under the following conditions on $q$, $\lambda$ and $\eta$, we have $L_T(\bar q, \bar \lambda) \le L_T(q, \bar \lambda) + \nu \text{ and }	L_T(\bar q, \bar \lambda) \ge L_T(\bar q, \lambda) - \nu$
	\begin{compactitem}
	\item If $\eta = \mathcal{O}(\sqrt{\frac{B^2T}{L^2 (\log T+ 1)}})$, $\nu = \mathcal{O}(\sqrt{\frac{B^2 L^2 (\log T + 1)}{T}})$; $\forall |\lambda| \le B$,  $\forall q \in \Delta$
	\item If $\eta = \frac{1}{T}$, $\nu = \mathcal{O}(\frac{L^2(\log T + 1)^2}{T})$; $\forall \lambda \in \R$, $\forall q \in \Delta$
\end{compactitem}
\end{theorem}
The above theorem indicates that the average play of the $q$-player and the
$\lambda$-player reaches a $\nu$-approximate saddle point. Our bounds for $\nu = \frac{1}{T}$ and $\lambda \in \R$ are strictly better than \cite{agarwal2018reductions}.\\

\begin{algorithm}[t]
	\caption{FairALM: DeepNet Classifier}
	\label{alg:deepnets_alg}
	\small
	\begin{algorithmic}[1]
		\STATE {\em Notations:} Dual step size $\eta$, Primal step size $\tau$
		\STATE {\em Input:}  Training Set $D$
		\STATE {\em Initializations:} $\lambda_0=0$,  $w_0$ 
		\FOR{$t = 0, 1, 2, ...,T$}
		\STATE Sample $z \sim D$
		\STATE  Pick $v_t \in \partial \Big(\hat e_{h_w}(z) + (\lambda_t + \eta) \hat \mu_{h_w}^{s_0} (z) - (\lambda_t-\eta) \hat \mu_{h_w}^{s_1} (z) \Big) $ 
		\STATE (Primal) $w_t \leftarrow w_{t-1} - \tau v_t$ \label{alg:supp_ht_update}
		\STATE (Dual) $\lambda_{t+1} \leftarrow \lambda_t + \eta (\hat \mu_{h_{w_t}}^{s_0}(z) - \hat \mu_{h_{w_t}}^{s_1}(z)) $ \label{alg:supp_lambda_t update}
		\ENDFOR
		\STATE {\em Output:} $w_T$
	\end{algorithmic}
\end{algorithm}

\subsection{Can we train Fair Deep Neural Networks by adapting Alg. \ref{alg:fair_alm_linear}?} 
The key difficulty from the analysis standpoint we face in extending these 
results to the deep networks setting is that the number of classifiers 
$|\mathcal{H}|$ may be exponential in number of nodes/layers. 
This creates a potential problem in computing Step~\ref{alg: ht_update} of Algorithm~\ref{alg:fair_alm_linear} -- if viewed mechanistically, is not practical since an epsilon net over the family $\mathcal{H}$ (representable by a neural network) is exponential in size. 
Interestingly, notice that 
we often use over-parameterized networks for learning. This is 
a useful fact here because it means that there exists 
a solution where 
$\argmin_i (e_{h_i} + \lambda_t d_{h_i})$ is $0$. While iterating through 
all $h_i$s will be intractable, we may still able to obtain a solution via standard 
stochastic gradient descent (SGD) procedures \cite{zhang2016understanding}. The only
unresolved question then is if we can do posterior inference and obtain classifiers that are ``fair''. 
It turns out that the above procedure provides us an 
approximation if we leverage two facts: first, SGD can find the minimum of $L(h,\lambda)$ with respect to $h$ and second, 
recent results show that SGD, in fact, performs variational inference, implying 
that the optimization provides an approximate posterior \cite{chaudhari2018stochastic}.  Having discussed the issue
of the exponential sized $|\mathcal{H}|$ -- for which we settle for an approximate posterior -- we make three additional adjustments
to the algorithm to make it suitable for training deep networks.
First, the non-differentiable indicator function $\mathbbm{1}[\cdot]$ is replaced with a smooth surrogate function (such as a logistic function). Second, as it is hard
to evaluate  $e_h/\mu_h^{s}$ due to unavailability of the true data distribution, we instead calculate their empirical estimates
$z = (x; y; s)$, and denote it by $\hat e_{h}(z)/\hat \mu_{h}^{s}(z)$. Third, by exchanging the ``$\max$'' and ``$\min$'' in \eqref{eq: randomized_almminmax}, we
obtain an objective that {\em upper-bounds} our current objective in \eqref{eq: randomized_almminmax}. This provides us with a closed-form solution to $\lambda$ thus reducing the minmax
objective to a single simpler minimization problem. We present the algorithm for deep neural network training
in Alg.~\ref{alg:deepnets_alg} and call it \textit{FairALM: DeepNet Classifier.}

\section{Experiments}
A central theme in our experiments is to assess whether our proposed algorithm, FairALM,  can indeed obtain 
meaningful fairness measure scores {\em without} compromising the test set performance.
We evaluate FairALM on a number of problems where the dataset reflects certain
inherent societal/stereotypical biases.
Our evaluations are also designed with a few additional goals in mind.\\\\
{\bf Overview.} Our {\bf first} experiment on the CelebA dataset seeks to predict the value of 
a label for a face image while controlling for certain protected attributes
(gender, age). 
We discuss how prediction of some labels is {\em unfair} in an unconstrained model
and contrast with our FairALM.
Next, we focus on the label where predictions are the most unfair and
present comparisons against methods available in the literature.
For our {\bf second} experiment, we use the ImSitu dataset where images
correspond to a situation (activities, verb).
Expectedly, some activities such as driving or cooking are more strongly
associated with a specific gender.
We inspect if an unconstrained model is {\em unfair} when we ask it to learn to predict 
two gender correlated activities/verbs.
Comparisons with baseline methods will help measure FairALM's strengths/weaknesses.
We can use heat map visualizations
to qualitatively interpret the value of adding fairness constraints. 
We threshold the heat-maps to get an understanding of a general behavior of the models.
Our  {\bf third} experiment addresses an important problem in
medical/scientific studies. Small sample sizes necessitate
pooling data from multiple sites or scanners \cite{zhou2018statistical},
but introduce a site or scanner
specific nuisance variable which must be controlled for --
else a deep (also, shallow) model may cheat and use site specific (rather than
disease-specific) artifacts in the images for prediction
even when the cohorts 
are age or gender matched \cite{inproceedings_fafp}. 
We study one simple setting here: we use FairALM 
to mitigate site (hospital) specific differences in predicting
``tuberculosis'' from X-ray images acquired at
two hospitals, Shenzhen and Montgomery (and recently made publicly available \cite{jaeger2014two}). 

In all the experiments, we impose Equality of Opportunity (EO) constraint (defined in Section~\ref{sec:EO}). We adopt NVP (novel validation procedure) used in \cite{donini2018empirical} to evaluate FairALM. It is a two-step procedure: first,
we search for the hyper-parameters that achieve the best accuracy, and then, we report the minimum fairness measure (DEO) for accuracies within $90\%$ of the highest accuracy. This offers some robustness of the reported numbers
to hyper-parameter selection. We describe these experiments one by one.\\\\
{\bf \textit{Remark from authors.} }
Certain attributes 
such as \textit{attractiveness}, obtained via crowd-sourcing,
may have socio-cultural ramifications.
Similarly, the gender attribute in the dataset is binary (male versus female)
which may be insensitive to some readers. 
We clarify that our goal is to present evidence showing that our algorithm
can impose fairness in a sensible way on datasets used in
the literature rather than the higher level question of whether our community needs to invest in 
culturally sensitive datasets with more societally relevant themes. 

\begin{table}[!t]
	\centering
	\resizebox{0.9\columnwidth}!{
		\begin{tabular}{cgg@{\hskip 0.2in}cgg@{\hskip 0.2in}}
			\multicolumn{3}{c}{Protected: \textbf{GENDER}} & \multicolumn{3}{c}{Protected: \textbf{YOUNG}} \\ 
			\cmidrule(lr){1-3}\cmidrule(lr){4-6} \morecmidrules	\cmidrule(lr){1-3}\cmidrule(lr){4-6}
			Label & \textbf{U} & \textbf{F} & Label & \textbf{U} & \textbf{F} \\
			\cmidrule(lr){1-3}\cmidrule(lr){4-6} \morecmidrules	\cmidrule(lr){1-3}\cmidrule(lr){4-6}
			Attractive & 28 & 3 & Attractive & 8 & 1 \\ \cmidrule(lr){1-3}\cmidrule(lr){4-6}
			Bangs & 4 & 2 & Heavy Makeup & 11 & 1 \\ \cmidrule(lr){1-3}\cmidrule(lr){4-6}
			High Cheekbones & 18 & 0 & High Cheekbones & 7 & 0 \\ \cmidrule(lr){1-3}\cmidrule(lr){4-6}
			Mouth Slightly open & 11 & 3 & Male & 6 & 0 \\ \cmidrule(lr){1-3}\cmidrule(lr){4-6}
			Smiling & 10 & 0 & Wearing Lipstick & 12 & 4 \\ \cmidrule(lr){1-3}\cmidrule(lr){4-6}			
			\cmidrule(lr){1-3}\cmidrule(lr){4-6} \morecmidrules	\cmidrule(lr){1-3}\cmidrule(lr){4-6}
		\end{tabular}
	}
	\caption{ \label{tab:celeba_ablation}\footnotesize \textbf{Identifying Unfair Labels in CelebA dataset.} We report the DEO measure for the Unconstrained model (\textbf{U}) and FairALM model (\textbf{F}). Using a 3-layers ReLU network, we determine the labels in CelebA dataset that are biased with respect to gender (left) and the attribute young (right). Labels with a precision of at least $70\%$ and a DEO of at least $4\%$ on the unconstrained model are reported here.}
	\vskip -0.2in
\end{table}

\subsection{CelebA dataset}
{\bf Data and Setup.} CelebA \cite{liu2018large} consists of $200$K celebrity face images from the internet annotated by a group of paid adult participants \cite{bohlen2017server}. There are up to $40$ labels available in the dataset, each of which is binary-valued.
\begin{figure}[!b]
	\centering
	\frame{
		\includegraphics[width=0.48\columnwidth]{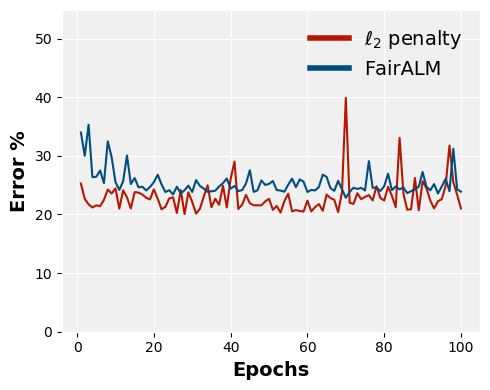}
		\hfill
		\includegraphics[width=0.48\columnwidth]{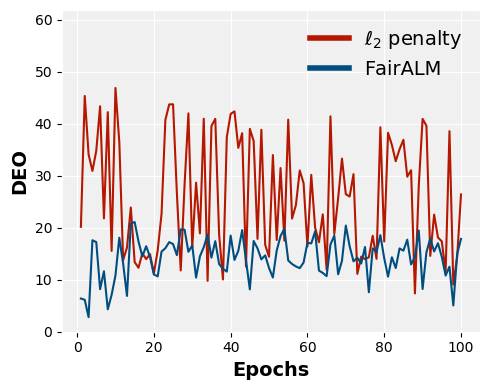}
	}
	\caption{\label{fig:l2_penalty} \footnotesize \textbf{Comparison to $\ell_2$ penalty .} FairALM has a stable training profile in comparison to naive $\ell_2$ penalty. The target label is \textit{attractiveness} and protected attribute is \textit{gender}. }
\end{figure}

{\bf Quantitative results.} We begin our analysis by predicting each of the $40$ labels with a $3$-layer ReLU network. The protected variable, $s$,
are the binary attributes like \textit{Male} and \textit{Young} representing gender and age respectively. We train the SGD algorithm for $5$-epochs and select the labels predicted with at least at $70\%$ precision and with a DEO of at least $4\%$ across the protected variables. The biased set of labels thus estimated are shown in Table~\ref{tab:celeba_ablation}. These labels are consistent with other reported results \cite{ryu2017inclusivefacenet}.
It is important to bear in mind that the bias in the labels should not be attributed to its relatedness to a specific protected attributed alone. The cause of bias could also be due to the skew in the label distributions. When training a $3$-layer ReLU net with FairALM, the precision of the model remained almost the same ($\pm 5\%$ ) while the DEO measure reduced significantly as indicated in the Table~\ref{tab:celeba_ablation}. Next, choosing the most unfair
label in Table~\ref{tab:celeba_ablation} (i.e., attractive), we train a ResNet18 for a longer duration of about $100$ epochs and contrast the performance with a simple $\ell_2$-penalty baseline. The training profile is observed to be more stable for FairALM as indicated in Fig.~\ref{fig:l2_penalty}. This finding is consistent with the seminal works such as \cite{bertsekas2014constrained,nocedal2006numerical} that discuss the ill-conditioned landscape of non-convex penalties. Comparisons to more recent works such as \cite{sattigeri2018fairness,quadrianto2019discovering} is provided in Table~\ref{tab:celeba_sota}. Here, we present a new state-of-the-art result for the DEO measure with the label \textit{attractive} and protected attribute \textit{gender}.
\begin{figure}[!t]
	\vskip -0.2in
	{\setlength{\fboxsep}{4pt}\fbox{
			\begin{minipage}{0.12\linewidth}
				\figuretitle{\colorbox{mylightgrey}{Gender}}
				\adjustbox{cfbox=mygrey 2pt 0pt}{\includegraphics[width=\linewidth]{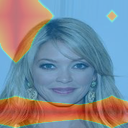}}
			\end{minipage}
			\begin{minipage}{0.12\linewidth}
				\figuretitle{\colorbox{mylightpink}{Unconstrained}}
				\adjustbox{cfbox=mypink 2pt 0pt}{\includegraphics[width=\linewidth]{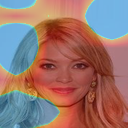}}
			\end{minipage}
			\begin{minipage}{0.12\linewidth}
				\figuretitle{\colorbox{mylightgreen}{FairALM}}
				\adjustbox{cfbox=mygreen 2pt 0pt}{\includegraphics[width=\linewidth]{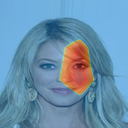}}
	\end{minipage}}}
	\hfill
	{\setlength{\fboxsep}{4pt}\fbox{
			\begin{minipage}{0.12\linewidth}
				\figuretitle{\colorbox{mylightgrey}{Gender}}
				\adjustbox{cfbox=mygrey 2pt 0pt}{\includegraphics[width=\linewidth]{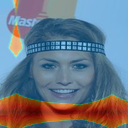}}
			\end{minipage}
			\begin{minipage}{0.12\linewidth}
				\figuretitle{\colorbox{mylightpink}{Unconstrained}}
				\adjustbox{cfbox=mypink 2pt 0pt}{\includegraphics[width=\linewidth]{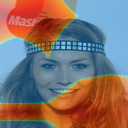}}
			\end{minipage}
			\begin{minipage}{0.12\linewidth}
				\figuretitle{\colorbox{mylightgreen}{FairALM}}
				\adjustbox{cfbox=mygreen 2pt 0pt}{\includegraphics[width=\linewidth]{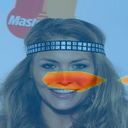}}
	\end{minipage}}}
	
	\vskip 4pt
	{\setlength{\fboxsep}{4pt}\fbox{
			\begin{minipage}{0.12\linewidth}
				\figuretitle{\colorbox{mylightgrey}{Gender}}
				\adjustbox{cfbox=mygrey 2pt 0pt}{\includegraphics[width=\linewidth]{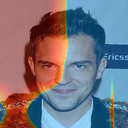}}
			\end{minipage}
			\begin{minipage}{0.12\linewidth}
				\figuretitle{\colorbox{mylightpink}{Unconstrained}}
				\adjustbox{cfbox=mypink 2pt 0pt}{\includegraphics[width=\linewidth]{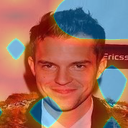}}
			\end{minipage}
			\begin{minipage}{0.12\linewidth}
				\figuretitle{\colorbox{mylightgreen}{FairALM}}
				\adjustbox{cfbox=mygreen 2pt 0pt}{\includegraphics[width=\linewidth]{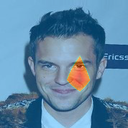}}
	\end{minipage}}}
	\hfill
	{\setlength{\fboxsep}{4pt}\fbox{
			\begin{minipage}{0.12\linewidth}
				\figuretitle{\colorbox{mylightgrey}{Gender}}
				\adjustbox{cfbox=mygrey 2pt 0pt}{\includegraphics[width=\linewidth]{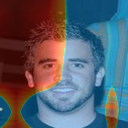}}
			\end{minipage}
			\begin{minipage}{0.12\linewidth}
				\figuretitle{\colorbox{mylightpink}{Unconstrained}}
				\adjustbox{cfbox=mypink 2pt 0pt}{\includegraphics[width=\linewidth]{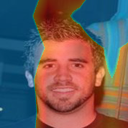}}
			\end{minipage}
			\begin{minipage}{0.12\linewidth}
				\figuretitle{\colorbox{mylightgreen}{FairALM}}
				\adjustbox{cfbox=mygreen 2pt 0pt}{\includegraphics[width=\linewidth]{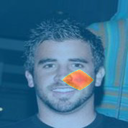}}
	\end{minipage}}}
	\caption{\footnotesize \textbf{Interpretable Models for CelebA.} Unconstrained/FairALM predict label \textit{attractiveness} while controlling \textit{gender}. The heatmaps of Unconstrained model overlaps with gender classification task indicating gender leak. FairALM consistently picks non-gender revealing features of the face. Interestingly, these regions are on the left side in accord with psychological studies that the Face's left side is more attractive \cite{Blackburn2012}.}
	\vskip -0.2in
\end{figure}
\begin{table}[!b]
	\centering
	\setlength{\tabcolsep}{5pt}
	\resizebox{0.8\columnwidth}{!}{
		\begin{tabular}[b]{cccc}
			\textbf{} & \begin{tabular}[c]{@{}c@{}}Fairness \\ GAN\cite{sattigeri2018fairness}\end{tabular} &
			\begin{tabular}[c]{@{}c@{}}Quadrianto \\ etal\cite{quadrianto2019discovering}\end{tabular} & \textbf{FairALM} \\ \hline\hline
			ERR & 26.6 & 24.1 & 24.5 \\ \hline
			\rowcolor[rgb]{0.89, 0.89, 1}DEO & 22.5 & 12.4 & \textbf{10.4} \\\hline
			FNR Female & 21.2 & 12.8 & \textbf{6.6} \\ \hline
			FNR Male & 43.7 & 25.2 & \textbf{17.0} \\ \hline\hline
	\end{tabular}}
	\vspace{12pt}
	\caption{\label{tab:celeba_sota}\footnotesize \textbf{Quantitative Results on CelebA.} FairALM attains a lower DEO measure and improves the testset errors (ERR). The target label is \textit{attractiveness} and protected attribute is \textit{gender}.}
\end{table}
\begin{figure*}[!t]
	\centering
	\begin{minipage}{\linewidth}
		\includegraphics[width=0.45\linewidth]{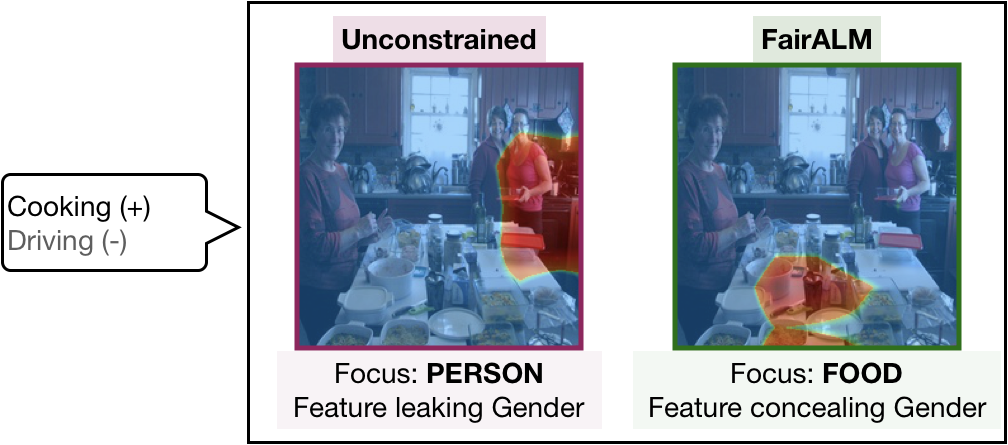} \hskip 8pt
		\includegraphics[width=0.45\linewidth]{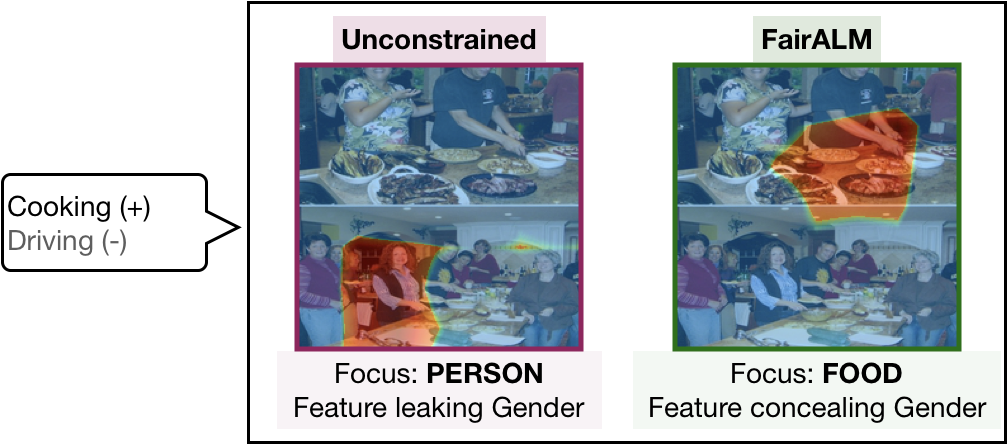}
	\end{minipage}%
	\vskip 4pt
	\begin{minipage}{\linewidth}
		\includegraphics[width=0.45\linewidth]{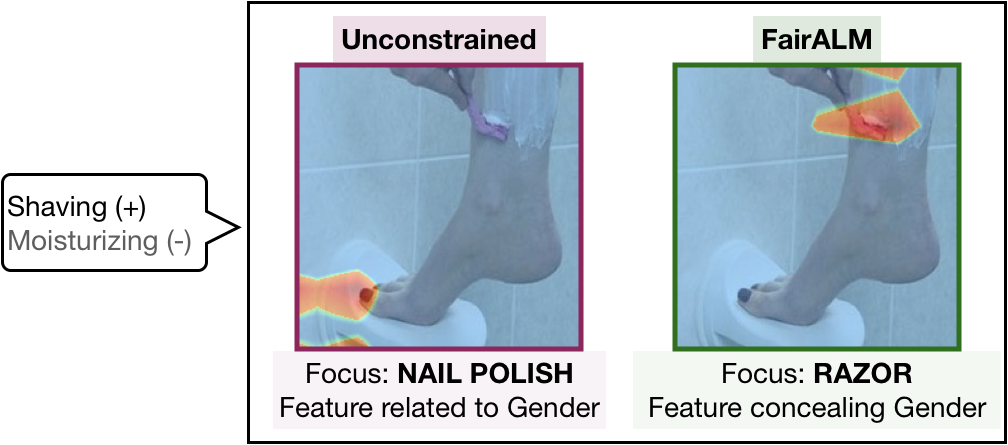} \hskip 8pt
		\includegraphics[width=0.45\linewidth]{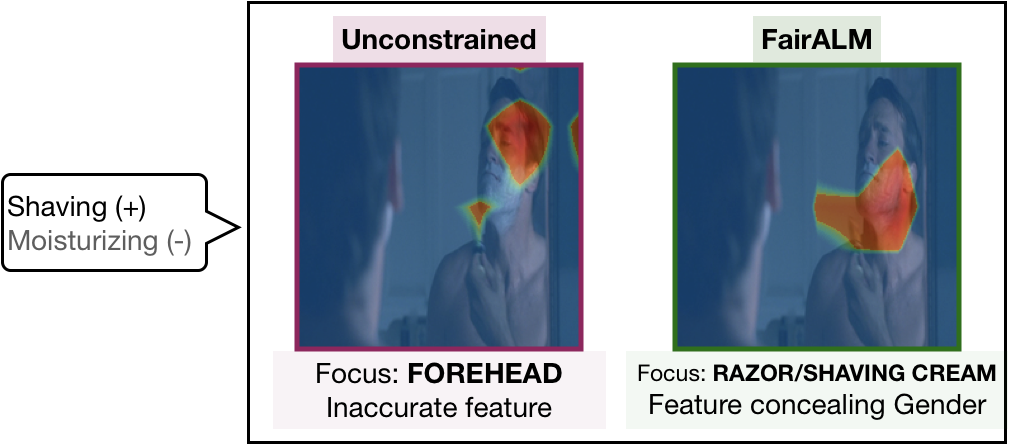}
	\end{minipage}
	\vskip 4pt
	\begin{minipage}{\linewidth}
		\includegraphics[width=0.45\linewidth]{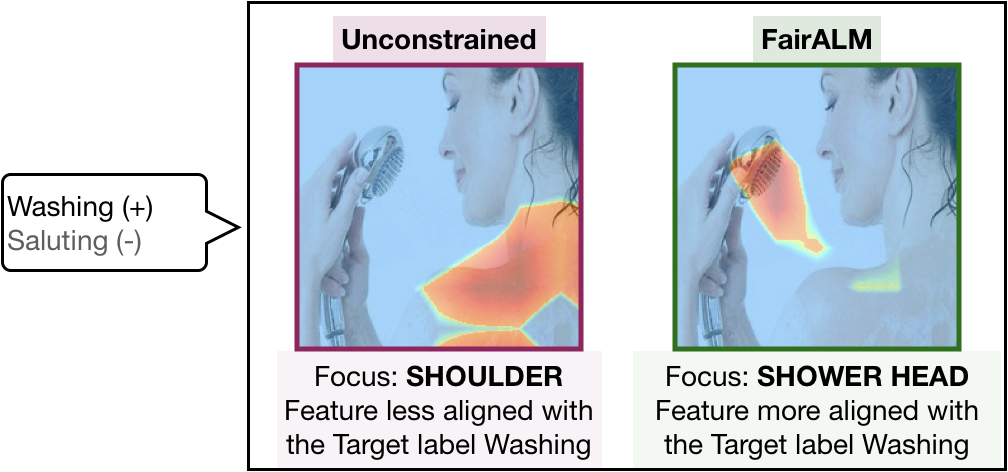} \hskip 8pt
		\includegraphics[width=0.45\linewidth]{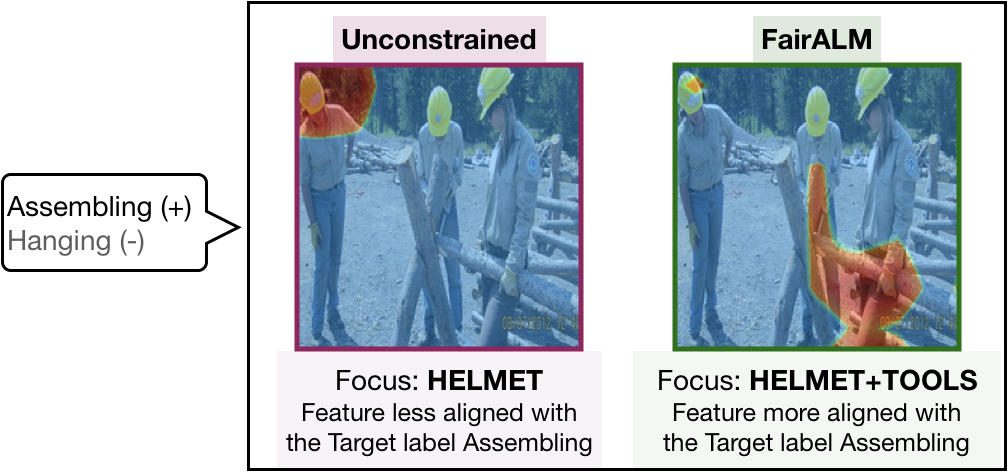}
	\end{minipage}
	\vskip -0.1in
	\caption{\label{fig:imsitu_interpretability} \footnotesize {\textbf{Interpretability in ImSitu.}} The activation maps indicate that FairALM conceals gender revealing attributes in an image. Moreover, the attributes are more aligned with label of interest. The target class predicted is indicated by a $+$. The activation maps in the examples shown in this figure are representative of the general behavior on this dataset. More examples can be found in the Appendix.}
	\vskip -0.1in
\end{figure*}

{\bf Qualitatively assessing Interpretability.} While the DEO measure obtained by FairALM is lower, we can ask an interesting question: when
we impose the fairness constraint, precisely which aspects of the image are no longer ``legal'' for
the neural network to utilize? This issue can be approached via
visualizing activation maps from models such as CAM \cite{DBLP:journals/corr/ZhouKLOT15}.
As a representative example, our analysis suggests that in general,
an unconstrained model uses the entire face image (including the gender-revealing
parts). We find some consistency between the activation maps for {\em attractiveness} and activation maps of an unconstrained model trained
to predict {\em gender}! 
In contrast, when we impose the fairness
constraint, the corresponding activation maps turn out to be clustered around specific
regions of the face which are {\em not} gender revealing.
In particular, a surprising finding was that the left regions
in the face were far more prominent which turns out to be consistent
with studies in psychology  \cite{Blackburn2012}.

{\bf Summary.} FairALM minimized the DEO measure without compromising the test error. It has a more stable training profile than an
$\ell_2$ penalty and is competitive with recent fairness methods in vision. The activation maps in FairALM concentrate on non-gender revealing features of the face when controlled for gender.

\subsection{Imsitu Dataset}
{\bf Data and Setup.} ImSitu \cite{yatskar2016} is a
           situation recognition dataset consisting of
           $\sim100$K color images taken from the web. The annotations for the image is provided as a summary of the activity in the image and includes a verb describing it, the interacting agents and their roles. The protected variable in this experiment is gender. 
Our objective is to classify a pair of verbs associated with an image. The pair is chosen such that if one of the verbs is biased towards males then the other would be biased towards females.
The authors in \cite{zhao2017men} report the list of labels in the ImSitu dataset that are gender
biased: we choose our verb pairs from this list. In particular, we consider the verbs \textit{Cooking vs Driving}, \textit{Shaving vs Moisturizing}, \textit{Washing vs Saluting} and \textit{Assembling vs Hanging}. We compare our results against multiple baselines such as 
	\begin{inparaenum}[\bfseries (1)]
		\item Unconstrained 
		\item \textit{$\ell_2$-penalty}, the penalty applied on the DEO measure
		\item \textit{Re-weighting}, a weighted loss functions where the weights account for the dataset skew
		\item \textit{Adversarial} \cite{zhang2018mitigating}
		\item \textit{Lagrangian} \cite{zhao2017men}
		\item \textit{Proxy-Lagrangian} \cite{cotter2018two}.
	\end{inparaenum}
The supplement includes more details of the baseline methods.
\begin{figure*}[!t]
	\centering
	\begin{subfigure}[b]{0.45\linewidth}
		\includegraphics[width=\linewidth]{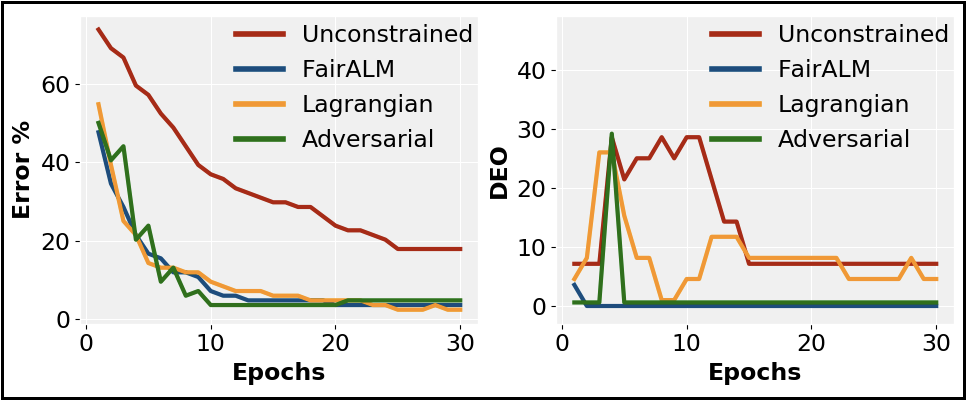}
		\caption{ \footnotesize Cooking {\tiny (+)} Driving {\tiny (-)}}
	\end{subfigure}%
	\hskip8pt
	\centering
	\begin{subfigure}[b]{0.45\linewidth}
		\includegraphics[width=\linewidth]{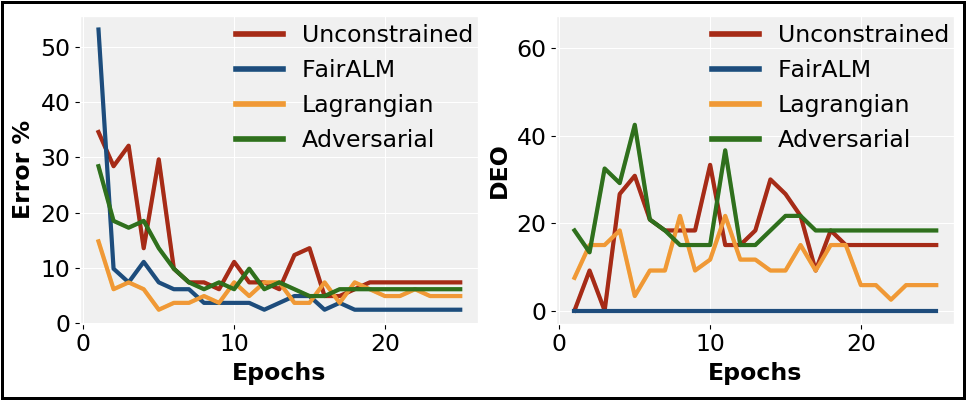}
		\caption{\footnotesize Assembling {\tiny (+)} Hanging {\tiny (-)}}
	\end{subfigure}%
	\vskip -0.1in
	\caption{\label{fig:imsitu_plots}\footnotesize \textbf{Training Profiles.} FairALM achieves minimum DEO early in training and remains competitive on testset errors. More plots in appendix.}
	\vskip -0.2in
\end{figure*} 

 {\bf Quantitative results.} From Fig.~\ref{fig:imsitu_plots}, it can be seen that FairALM reaches a zero DEO measure very early in training and attains better test errors than an unconstrained model. Within the family of Lagrangian methods such as \cite{zhao2017men,cotter2018two}, FairALM performs better on verb pair `Shaving vs Moisturizing' in both test error and DEO measure as indicated in Table~\ref{tab:imsitu-accs} . While the results on the other verb pairs are comparable, FairALM was observed to be more stable to different hyper-parameter choices. This finding is in accord with recent studies by \cite{asi2019stochastic} who prove that proximal function models are robust to step-size selection. Detailed analysis is provided in the supplement. Turning now to an adversarial method such as \cite{zhao2017men}, results in Table~\ref{tab:imsitu-accs} show that the DEO measure is not controlled as competently as FairALM. Moreover, complicated training routines and unreliable convergence \cite{barnett2018convergence} makes model-training harder.
 
{\bf Interpretable Models.} We used CAM \cite{DBLP:journals/corr/ZhouKLOT15} to inspect the image regions used by the
model for target prediction. We observe that the unconstrained model ends up picking features from locations that may not be relevant for the task description 
but merely co-occur with the verbs in this particular dataset (and are gender-biased). 
Fig.~\ref{fig:imsitu_interpretability} highlights this observation for the selected classification tasks. Overall, we observe that the semantic regions used by the constrained model are more aligned with the action verb present in the image, and this adds to the qualitative advantages of the model trained using FairALM in terms of interpretability.

{\bf Limitations}. We also note that there are cases
where both the unconstrained model and FairALM
look at incorrect image regions for prediction, owing to the small dataset sizes.
However, the number of such cases are far fewer for FairALM than the unconstrained setup.

{\bf Summary}. FairALM successfully minimizes the fairness measure while classifying verb/action pairs associated with an image. FairALM uses regions in an image that are more relevant to the target class and less gender revealing.
\begin{table}[!t]
	\centering
	\resizebox{\columnwidth}{!}{%
		\begin{tabular}{c@{\hskip 0.1in}cg@{\hskip 0.25in}cg@{\hskip 0.25in}cg@{\hskip 0.25in}cg}
			\textbf{} & \multicolumn{2}{c}{\begin{tabular}[c]{@{}c@{}}Cooking{\tiny{(+)}} \\ Driving{\tiny{(-)}}\end{tabular}} & \multicolumn{2}{c}{\begin{tabular}[c]{@{}c@{}}Shaving{\tiny{(+)}} \\ Moisturize{\tiny{(-)}}\end{tabular}} & \multicolumn{2}{c}{\begin{tabular}[c]{@{}c@{}}Washing{\tiny{(+)}} \\ Saluting{\tiny{(-)}}\end{tabular}} & \multicolumn{2}{c}{\begin{tabular}[c]{@{}c@{}}Assembling{\tiny{(+)}} \\ Hanging{\tiny{(-)}}\end{tabular}} \\ \hline\hline
			& ERR & DEO & ERR & DEO & ERR & DEO & ERR & DEO \\ \hline\hline
			\begin{tabular}[c]{@{}c@{}}Unconstrained\end{tabular} & 17.9 & 7.1 & 23.6 & 4.2 & 12.8 & 25.9 & 7.5 & 15.0 \\ \hline
			\begin{tabular}[c]{@{}c@{}}$\ell_2$ Penalty\end{tabular} & 14.3 & 14.0 & 23.6 & 1.3 & 10.9 & 0.0 & 5.0 & 21.6 \\ \hline
			Reweight & 11.9 & 3.5 & 19.0 & 5.3 & 10.9 & 0.0 & 4.9 & 9.0 \\ \hline
			Adversarial & 4.8 & 0.0 & 13.5 & 11.9 & 14.6 & 25.9 & 6.2 & 18.3 \\ \hline
			Lagrangian & 2.4 & 3.5 & 12.4 & 12.0 & 3.7 & 0.0 & 5.0 & 5.8\\ \hline
			Proxy-lagragn. & 2.4 & 3.5 & 12.4 & 12.0 & 3.7 & 0.0 & 14.9 & 26.0 \\ \hline
			\textbf{FairALM} & 3.6 & 0.0 & 20.0 & 0.0 & 7.3 & 0.0 & 2.5 & 0.0 \\ \hline\hline
		\end{tabular}%
	}
	\caption{\label{tab:imsitu-accs} \footnotesize \textbf{Quantitative Results on ImSitu.} Test errors (ERR) and DEO measure are reported in $\%$. The target class that is to be predicted in is indicated by a $+$. FairALM always achieves a zero DEO while remaining competitive in ERR with the best method for a given verb-pair.}
	\vskip -0.2in
\end{table}

\subsection{Chest X-Ray datasets}
{\bf Data and Setup.} The datasets we examine here are publicly available from
        the U.S. National Library of Medicine \cite{jaeger2014two}. The images come from two sites/sources.
        Images for the first site are collected from patients in Montgomery county, USA and
        includes $138$ x-rays. The second set of images includes
        $662$ images collected at a hospital in Shenzhen, China.
        Our task is to predict pulmonary tuberculosis (TB) from the x-ray images.
        The images are collected from different x-ray machines with
        different characteristics, and have site-specific markings or artifacts, see
        Fig~\ref{fig:chestxray_concept}.  $25\%$ of the samples from the pooked dataset are set aside for testing.
        \begin{figure}[!b]
        	\centering
        	\includegraphics[width=0.5\columnwidth]{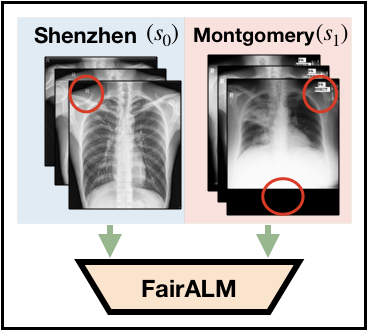}
        	\caption{\label{fig:chestxray_concept} \footnotesize \textbf{FairALM for dataset pooling.} Data is pooled from two sites/hospitals, Shenzhen $s_0$ and Montgomery $s_1$. }
        	\vskip-5pt
        \end{figure}
           \begin{figure}[!b]
        	\centering
        	\frame{
        		\includegraphics[width=0.3\columnwidth]{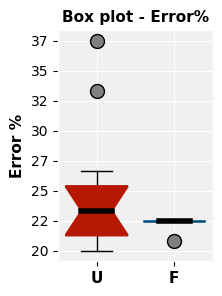}
        		\includegraphics[width=0.3\columnwidth]{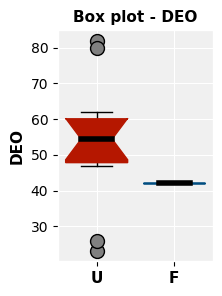}
        		\includegraphics[width=0.3\columnwidth]{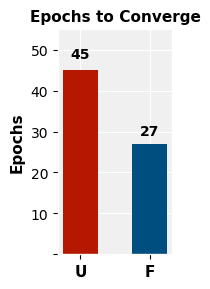}
        	}
        	\caption{\label{fig:chestxray} \footnotesize \textbf{Better Generalization with FairALM.} We compare Unconstrained mode (\textbf{U}) and FairALM (\textbf{F}) Box-plots indicate a lower variance in testset error and the DEO measure for FairALM. Moreover, FairALM reaches $20\%$ testset error in fewer epochs.}
        	\vskip-5pt
        \end{figure}
        
        {\bf Quantitative Results.} We treat the site information, Montgomery or Shenzhen, as a nuisance/protected variable and seek to decorrelate it from the TB labels. We train a ResNet18 network and compare an unconstrained model with FairALM model. Our datasets of choice are small in size, and so deep models easily overfit to site-specific biases present in the training data. Our results corroborate this conjecture, the training accuracies reach $100\%$ very early and the test set accuracies for the unconstrained model has a large variance over multiple experimental runs. Conversely, as depicted in Fig.~\ref{fig:chestxray}, a FairALM model not only maintains a lower variance in the test set errors and DEO measure but also attains
        improved performance on these measures. What stands out in this experiment is that the number of
        epochs to reach a certain test set error is lower for FairALM indicating that the model generalizes faster compared to an unconstrained model.
               
{\bf Summary.} FairALM is effective at learning from datasets from two different sites/sources, minimizes site-specific biases and accelerates generalization.

\section{Conclusion}
  We introduced FairALM, an augmented Lagrangian framework to impose constraints on
  fairness measures studied in the literature. On the theoretical side, we provide
  strictly better bounds -- $\mathcal{O}\bigg(\frac{\log^2 T}{T}\bigg)$
  versus $\mathcal{O}\bigg(\frac{1}{\sqrt{T}}\bigg)$, for reaching a saddle point.
  On the application side, we provide extensive evidence (qualitative and quantitative) on image
  datasets commonly used in vision to show the potential benefits of our proposal.
  Finally, we use FairALM to mitigate site specific differences when performing analysis of
  pooled medical image datasets. In applying deep learning to scientific/biomedical problems,
  this is an important issue since sample sizes at individual sites/institutions are often
  smaller.
  The overall procedure is simple which we believe will lead to broader adoption
  and follow-up work on this socially relevant topic.

{\small
\bibliographystyle{ieee_fullname}
\bibliography{egbib}

\begin{thebibliography}{10}\itemsep=-1pt

\bibitem{agarwal2018reductions}
Alekh Agarwal, Alina Beygelzimer, Miroslav Dud{\'\i}k, John Langford, and Hanna
  Wallach.
\newblock A reductions approach to fair classification.
\newblock {\em arXiv preprint arXiv:1803.02453}, 2018.

\bibitem{asi2019stochastic}
Hilal Asi and John~C Duchi.
\newblock Stochastic (approximate) proximal point methods: Convergence,
  optimality, and adaptivity.
\newblock {\em SIAM Journal on Optimization}, 29(3):2257--2290, 2019.

\bibitem{barnett2018convergence}
Samuel~A Barnett.
\newblock Convergence problems with generative adversarial networks (gans).
\newblock {\em arXiv preprint arXiv:1806.11382}, 2018.

\bibitem{bechavod2017penalizing}
Yahav Bechavod and Katrina Ligett.
\newblock Penalizing unfairness in binary classification.
\newblock {\em arXiv preprint arXiv:1707.00044}, 2017.

\bibitem{bertsekas2014constrained}
Dimitri~P Bertsekas.
\newblock {\em Constrained optimization and Lagrange multiplier methods}.
\newblock Academic press, 2014.

\bibitem{Blackburn2012}
Kelsey Blackburn and James Schirillo.
\newblock Emotive hemispheric differences measured in real-life portraits using
  pupil diameter and subjective aesthetic preferences.
\newblock {\em Experimental Brain Research}, 219(4):447--455, Jun 2012.

\bibitem{bohlen2017server}
Marc B{\"o}hlen, Varun Chandola, and Amol Salunkhe.
\newblock Server, server in the cloud. who is the fairest in the crowd?
\newblock {\em arXiv preprint arXiv:1711.08801}, 2017.

\bibitem{bolukbasi2016man}
Tolga Bolukbasi, Kai-Wei Chang, James~Y Zou, Venkatesh Saligrama, and Adam~T
  Kalai.
\newblock Man is to computer programmer as woman is to homemaker? debiasing
  word embeddings.
\newblock In {\em Advances in neural information processing systems}, pages
  4349--4357, 2016.

\bibitem{buolamwini2018gender}
Joy Buolamwini and Timnit Gebru.
\newblock Gender shades: Intersectional accuracy disparities in commercial
  gender classification.
\newblock In {\em Conference on fairness, accountability and transparency},
  pages 77--91, 2018.

\bibitem{calmon2017optimized}
Flavio Calmon, Dennis Wei, Bhanukiran Vinzamuri, Karthikeyan~Natesan
  Ramamurthy, and Kush~R Varshney.
\newblock Optimized pre-processing for discrimination prevention.
\newblock In {\em Advances in Neural Information Processing Systems}, pages
  3992--4001, 2017.

\bibitem{celis2019classification}
L~Elisa Celis, Lingxiao Huang, Vijay Keswani, and Nisheeth~K Vishnoi.
\newblock Classification with fairness constraints: A meta-algorithm with
  provable guarantees.
\newblock In {\em Proceedings of the Conference on Fairness, Accountability,
  and Transparency}, pages 319--328. ACM, 2019.

\bibitem{chaudhari2018stochastic}
Pratik Chaudhari and Stefano Soatto.
\newblock Stochastic gradient descent performs variational inference, converges
  to limit cycles for deep networks.
\newblock In {\em 2018 Information Theory and Applications Workshop (ITA)}.
  IEEE, 2018.

\bibitem{chin_2019}
Caitlin Chin.
\newblock Assessing employer intent when ai hiring tools are biased, Dec 2019.

\bibitem{chouldechova2017fair}
Alexandra Chouldechova.
\newblock Fair prediction with disparate impact: A study of bias in recidivism
  prediction instruments.
\newblock {\em Big data}, 5(2):153--163, 2017.

\bibitem{cotter2018two}
Andrew Cotter, Heinrich Jiang, and Karthik Sridharan.
\newblock Two-player games for efficient non-convex constrained optimization.
\newblock {\em arXiv preprint arXiv:1804.06500}, 2018.

\bibitem{cotter2018optimization}
Andrew Cotter, Heinrich Jiang, Serena Wang, Taman Narayan, Maya Gupta, Seungil
  You, and Karthik Sridharan.
\newblock Optimization with non-differentiable constraints with applications to
  fairness, recall, churn, and other goals.
\newblock {\em arXiv preprint arXiv:1809.04198}, 2018.

\bibitem{donini2018empirical}
Michele Donini, Luca Oneto, Shai Ben-David, John~S Shawe-Taylor, and
  Massimiliano Pontil.
\newblock Empirical risk minimization under fairness constraints.
\newblock In {\em Advances in Neural Information Processing Systems}, pages
  2791--2801, 2018.

\bibitem{duchi2012randomized}
John~C Duchi, Peter~L Bartlett, and Martin~J Wainwright.
\newblock Randomized smoothing for stochastic optimization.
\newblock {\em SIAM Journal on Optimization}, 22(2):674--701, 2012.

\bibitem{inproceedings_fafp}
Alhussein Fawzi and Pascal Frossard.
\newblock Measuring the effect of nuisance variables on classifiers.
\newblock pages 137.1--137.12, 01 2016.

\bibitem{fish2016confidence}
Benjamin Fish, Jeremy Kun, and {\'A}d{\'a}m~D Lelkes.
\newblock A confidence-based approach for balancing fairness and accuracy.
\newblock In {\em Proceedings of the 2016 SIAM International Conference on Data
  Mining}, pages 144--152. SIAM, 2016.

\bibitem{goh2016satisfying}
Gabriel Goh, Andrew Cotter, Maya Gupta, and Michael~P Friedlander.
\newblock Satisfying real-world goals with dataset constraints.
\newblock In {\em Advances in Neural Information Processing Systems}, pages
  2415--2423, 2016.

\bibitem{Goodman_Flaxman_2017}
Bryce Goodman and Seth Flaxman.
\newblock European union regulations on algorithmic decision-making and a
  “right to explanation”.
\newblock {\em AI Magazine}, 38(3):50--57, Oct. 2017.

\bibitem{hardt2016equality}
Moritz Hardt, Eric Price, Nati Srebro, et~al.
\newblock Equality of opportunity in supervised learning.
\newblock In {\em Advances in neural information processing systems}, pages
  3315--3323, 2016.

\bibitem{heilweil_2019}
Rebecca Heilweil.
\newblock Artificial intelligence will help determine if you get your next job,
  Dec 2019.

\bibitem{jaeger2014two}
Stefan Jaeger, Sema Candemir, Sameer Antani, Y{\`\i}-Xi{\'a}ng~J W{\'a}ng,
  Pu-Xuan Lu, and George Thoma.
\newblock Two public chest x-ray datasets for computer-aided screening of
  pulmonary diseases.
\newblock {\em Quantitative imaging in medicine and surgery}, 4(6):475, 2014.

\bibitem{kamiran2010classification}
Faisal Kamiran and Toon Calders.
\newblock Classification with no discrimination by preferential sampling.
\newblock In {\em Proc. 19th Machine Learning Conf. Belgium and The
  Netherlands}, pages 1--6. Citeseer, 2010.

\bibitem{kearns2017preventing}
Michael Kearns, Seth Neel, Aaron Roth, and Zhiwei~Steven Wu.
\newblock Preventing fairness gerrymandering: Auditing and learning for
  subgroup fairness.
\newblock {\em arXiv preprint arXiv:1711.05144}, 2017.

\bibitem{liu2018large}
Ziwei Liu, Ping Luo, Xiaogang Wang, and Xiaoou Tang.
\newblock Large-scale celebfaces attributes (celeba) dataset.
\newblock {\em Retrieved August}, 15:2018, 2018.

\bibitem{nocedal2006numerical}
Jorge Nocedal and Stephen Wright.
\newblock {\em Numerical optimization}.
\newblock Springer Science \& Business Media, 2006.

\bibitem{parikh2014proximal}
Neal Parikh and Stephen Boyd.
\newblock Proximal algorithms.
\newblock {\em Foundations and Trends in optimization}, 1(3):127--239, 2014.

\bibitem{quadrianto2019discovering}
Novi Quadrianto, Viktoriia Sharmanska, and Oliver Thomas.
\newblock Discovering fair representations in the data domain.
\newblock In {\em Proceedings of the IEEE Conference on Computer Vision and
  Pattern Recognition}, pages 8227--8236, 2019.

\bibitem{ryu2017inclusivefacenet}
Hee~Jung Ryu, Hartwig Adam, and Margaret Mitchell.
\newblock Inclusivefacenet: Improving face attribute detection with race and
  gender diversity.
\newblock {\em arXiv preprint arXiv:1712.00193}, 2017.

\bibitem{sattigeri2018fairness}
Prasanna Sattigeri, Samuel~C Hoffman, Vijil Chenthamarakshan, and Kush~R
  Varshney.
\newblock Fairness gan.
\newblock {\em arXiv preprint arXiv:1805.09910}, 2018.

\bibitem{shalev2012online}
Shai Shalev-Shwartz et~al.
\newblock Online learning and online convex optimization.
\newblock {\em Foundations and Trends{\textregistered} in Machine Learning},
  4(2):107--194, 2012.

\bibitem{ustun2016learning}
Berk Ustun and Cynthia Rudin.
\newblock Learning optimized risk scores from large-scale datasets.
\newblock {\em stat}, 1050:1, 2016.

\bibitem{woodworth2017learning}
Blake Woodworth, Suriya Gunasekar, Mesrob~I Ohannessian, and Nathan Srebro.
\newblock Learning non-discriminatory predictors.
\newblock {\em arXiv preprint arXiv:1702.06081}, 2017.

\bibitem{yao2017beyond}
Sirui Yao and Bert Huang.
\newblock Beyond parity: Fairness objectives for collaborative filtering.
\newblock In {\em Advances in Neural Information Processing Systems}, pages
  2921--2930, 2017.

\bibitem{yatskar2016}
Mark Yatskar, Luke Zettlemoyer, and Ali Farhadi.
\newblock Situation recognition: Visual semantic role labeling for image
  understanding.
\newblock In {\em Conference on Computer Vision and Pattern Recognition}, 2016.

\bibitem{zafar2017fairness}
Muhammad~Bilal Zafar, Isabel Valera, Manuel Gomez~Rodriguez, and Krishna~P
  Gummadi.
\newblock Fairness beyond disparate treatment \& disparate impact: Learning
  classification without disparate mistreatment.
\newblock In {\em Proceedings of the 26th International Conference on World
  Wide Web}, pages 1171--1180. International World Wide Web Conferences
  Steering Committee, 2017.

\bibitem{zafar2017parity}
Muhammad~Bilal Zafar, Isabel Valera, Manuel Rodriguez, Krishna Gummadi, and
  Adrian Weller.
\newblock From parity to preference-based notions of fairness in
  classification.
\newblock In {\em Advances in Neural Information Processing Systems}, pages
  229--239, 2017.

\bibitem{zhang2018mitigating}
Brian~Hu Zhang, Blake Lemoine, and Margaret Mitchell.
\newblock Mitigating unwanted biases with adversarial learning.
\newblock In {\em Proceedings of the 2018 AAAI/ACM Conference on AI, Ethics,
  and Society}, pages 335--340, 2018.

\bibitem{zhang2016understanding}
Chiyuan Zhang, Samy Bengio, Moritz Hardt, Benjamin Recht, and Oriol Vinyals.
\newblock Understanding deep learning requires rethinking generalization.
\newblock {\em arXiv preprint arXiv:1611.03530}, 2016.

\bibitem{zhao2019gender}
Jieyu Zhao, Tianlu Wang, Mark Yatskar, Ryan Cotterell, Vicente Ordonez, and
  Kai-Wei Chang.
\newblock Gender bias in contextualized word embeddings.
\newblock {\em arXiv preprint arXiv:1904.03310}, 2019.

\bibitem{zhao2017men}
Jieyu Zhao, Tianlu Wang, Mark Yatskar, Vicente Ordonez, and Kai-Wei Chang.
\newblock Men also like shopping: Reducing gender bias amplification using
  corpus-level constraints.
\newblock {\em arXiv preprint arXiv:1707.09457}, 2017.

\bibitem{DBLP:journals/corr/ZhouKLOT15}
Bolei Zhou, Aditya Khosla, {\`{A}}gata Lapedriza, Aude Oliva, and Antonio
  Torralba.
\newblock Learning deep features for discriminative localization.
\newblock {\em CoRR}, abs/1512.04150, 2015.

\bibitem{zhou2018statistical}
Hao~Henry Zhou, Vikas Singh, Sterling~C Johnson, Grace Wahba, Alzheimer’s
  Disease~Neuroimaging Initiative, et~al.
\newblock Statistical tests and identifiability conditions for pooling and
  analyzing multisite datasets.
\newblock {\em Proceedings of the National Academy of Sciences},
  115(7):1481--1486, 2018.

\bibitem{zuber2014critical}
Ortrun Zuber-Skerritt and Eva Cendon.
\newblock Critical reflection on professional development in the social
  sciences: interview results.
\newblock {\em International Journal for Researcher Development}, 5(1):16--32,
  2014.

\end{thebibliography}
}

\onecolumn

\section{Appendix}

\subsection{Experiments on \textit{FairALM: Linear Classifier}}
{\bf Data.} We consider four standard datasets, {\tt Adult},
{\tt COMPAS}, {\tt German} and {\tt Law Schools} \cite{donini2018empirical,agarwal2018reductions}. The {\tt Adult} dataset
is comprised of demographic characteristics where the task
is to predict if a person has an income higher
(or lower) than $\$50$K per year. The protected attribute here is gender.
In {\tt COMPAS} dataset, the task is to predict the recidivism of individuals based on features such as age, gender, race, prior offenses and charge degree. The protected attribute
here is race, specifically, whether the individual is white
or black. The {\tt German} dataset classifies people as good or bad credit risks with the person being a foreigner or not as the protected attribute. The features available in this dataset are credit history, saving
accounts, bonds, etc. Finally, the {\tt Law Schools} dataset, which comprises of $\sim20$K examples, seeks to predict a person's passage of the bar exam. Here, a binary attribute race is considered as the protected attribute.\\

{\bf Setup.} We use Alg.~$1$ in the paper for experiments in this section. Recall from $\S~3$ of the paper that Alg.~$1$ requires the specification of $\mathcal{H}$. We use the space of logistic regression classifiers as $\mathcal{H}$. At the start of the algorithm we have an empty set of classifiers. In each iteration, we add a newly trained classifier $h \in \mathcal{H}$ to the set of classifiers only if $h$ has a smaller Lagrangian objective value among all the classifiers already in the set. \\

{\bf Quantitative Results.} For the {\tt Adult} dataset, FairALM attains a smaller test error and smaller DEO compared to the baselines considered in Table~\ref{tab:linear_classifier}. We see big improvements on the DEO measure in {\tt COMPAS} dataset and test error in {\tt German} dataset using FairALM. While the performance of FairALM on {\tt Law Schools} is comparable to other methods, it obtains a better false-positive rate than \cite{agarwal2018reductions} which is a better metric as this dataset is skewed towards it's target class.\\

{\bf Summary.} We train Alg.~$1$ on standard datasets specified in \cite{donini2018empirical,agarwal2018reductions}. We observe that FairALM is competitive with the popular methods in the fairness literature.

\begin{table}[!h]
	\centering
	\resizebox{0.9\linewidth}{!}{%
		\begin{tabular}{c@{\hskip 0.2in}cg@{\hskip 0.25in}cg@{\hskip 0.25in}cg@{\hskip 0.25in}cg} 
			& \multicolumn{2}{c}{Adult} & \multicolumn{2}{c}{COMPAS} & \multicolumn{2}{c}{German} & \multicolumn{2}{c}{Law Schools} \\ \hline\hline
			& ERR & DEO & ERR & DEO & ERR & DEO & ERR & DEO \\ \hline\hline
			Zafar \textit{et al.} \cite{zafar2017fairness} & $22.0$ & $5.0$ & $31.0$ & $10.0$ & $38.0$ & $13.0$ & $-$ & $-$ \\ \hline
			Hardt \textit{et al.} \cite{hardt2016equality} & $18.0$ & $11.0$ & $29.0$ & $8.0$ & $29.0$ & $11.0$ & $4.5$ & $0.0$ \\ \hline
			Donini \textit{et al.} \cite{donini2018empirical} & $19.0$ & $1.0$ & $27.0$ & $5.0$ & $27.0$ & $5.0$ & $-$ & $-$ \\ \hline
			Agarwal \textit{et al.} \cite{agarwal2018reductions} & $17.0$ & $1.0$ & $31.0$ & $3.0$ & $-$ & $-$ & $4.5$ & $1.0$ \\ \hline
			\textbf{FairALM} & $15.8 \pm 1$ & $0.7 \pm 0.6$& $34.7 \pm 1$& $0.1 \pm 0.1$ & $24.3 \pm 2.7$ & $10.8 \pm 4.5$ &$4.8 \pm 0.1$ & $0.4 \pm 0.2$ \\ \hline\hline
		\end{tabular}%
	}
	\caption{\label{tab:linear_classifier} \footnotesize \textbf{Standard Datasets.} We report test error (ERR) and DEO fairness measure in $\%$.  FairALM attains minimal DEO measure among the baseline methods while maintaining a similar test error.}
\end{table}

\subsection{Proofs for theoretical claims in the paper}
Prior to proving the convergence of primal and dual variables of our algorithm with respect to the augmented lagrangian $L_T(q, \lambda)$, we prove a regret bound on the function $f_t(\lambda)$ which is defined in the following lemma. As $f_t(\lambda)$ is a strongly concave function (which we shall see shortly), we obtain a bound on the negative regret.
\begin{lemma}
	\label{lemma:regretbound}
	Let $r_t$ denote the reward at each round of the game. The reward function $f_t(\lambda)$ is defined as $f_t(\lambda) = \lambda r_t - \frac{1}{2\eta} (\lambda - \lambda_t)^2$. We choose $\lambda$ in the round $T+1$ to maximize the cumulative reward, i.e., $\lambda_{T+1} = \argmax_{\lambda} \sum_{t=1}^T f_t(\lambda)$.  Define $L = \max_t \mid r_t\mid$.   We obtain the following bound on the cumulative reward, for any $\lambda$,
	\begin{align}
	\sum_{t=1}^T \bigg( \lambda r_t - \frac{1}{2\eta} (\lambda - \lambda_t)^2 \bigg) \le \sum_{t=1}^T \lambda_t r_t + \eta L^2 \mathcal{O}(\log T)
	\end{align}	
\end{lemma}
\begin{proof}
	As we are maximizing the cumulative reward function, in the $(t+1)^{th}$ iteration $\lambda_{t+1}$ is updated as $\lambda_{t+1} = \argmax_{\lambda} \sum_{i=1}^t f_i(\lambda)$. This learning rule is also called the Follow-The-Leader  (FTL) principle which is discussed in Section $2.2$ of \cite{shalev2012online}. Emulating the proof of Lemma $2.1$ in \cite{shalev2012online}, a bound on the negative regret of FTL, for any $\lambda \in \R$, can be derived due to the concavity of $f_t(\lambda)$,
	\begin{align}
		\label{suppeq_lemma21}
		\sum_{t=1}^T f_t(\lambda) - \sum_{t=1}^{T} f_t(\lambda_t) \le \sum_{t=1}^{T} f_t(\lambda_{t+1}) - \sum_{t=1}^{T} f_t(\lambda_t)
	\end{align}
	Our objective, now, is to obtain a bound on RHS of \eqref{suppeq_lemma21}. Solving {\tt $\argmax_{\lambda} \sum_{i=1}^t f_i(\lambda)$} for $\lambda$ will show us how $\lambda_t$ and $\lambda_{t+1}$ are related,
	\begin{align}
		\lambda_{t+1 } = \frac{\eta}{t} \sum_{i=1}^{t}r_i + \frac{1}{t} \sum_{i=1}^{t}\lambda_i 
		\label{step:lambda_relation}
		\hskip 8pt \implies \lambda_{t+1} - \lambda_{t} = \frac{\eta}{t}r_t
	\end{align}
	Using \eqref{step:lambda_relation}, we obtain a bound on $f_t(\lambda_{t+1}) - f_t(\lambda_t)$, we have,
	\begin{align*}
		f_t(\lambda_{t+1}) - f_t(\lambda_t) &\le \frac{\eta}{t}r_t^2
	\end{align*}
	With $L = \max_t |r_t|$ and using the fact that $\sum_{i=1}^{T} \frac{1}{i} \le (\log T + 1)$,
	\begin{align}
		\label{suppeq:cum1}
		\sum_{t=1}^{T} \Big(f_t(\lambda_{t+1}) - f_t(\lambda_{t}) \Big) \le \eta L^2 (\log T + 1)
	\end{align}
	Let us denote $\xi_T = \eta L^2 (\log T + 1)$, we bound \eqref{suppeq_lemma21} with \eqref{suppeq:cum1},
	\begin{empheq}[box={\Garybox[Cumulative Reward Bound]}]{align}
		\label{suppeq:paper_lemma}
		\forall \lambda \in \R \quad \sum_{t=1}^{T} \Big( \lambda r_t - \frac{1}{2\eta}(\lambda-\lambda_t)^2 \Big) \le \Big(\sum_{t=1}^{T} \lambda_t r_t\Big) + \xi_T
	\end{empheq}
\end{proof}
Next, using the \textit{Cumulative Reward Bound}~\eqref{suppeq:paper_lemma}, we prove the theorem stated in the paper.  The theorem gives us the number of iterations required by Alg.~$1$ (in the paper) to reach a $\nu-$approximate saddle point. Our bounds for $\eta=\frac{1}{T}$ and $\lambda \in \R$ are strictly better than $\cite{agarwal2018reductions}$. We re-state the theorem here,
\begin{theorem}
	Recall that $d_h$ represents the difference of conditional means. Assume that $|| d_h||_{\infty} \le L$ and consider $T$ rounds of Alg~$1$ (in the paper). Let $\bar q := \frac{1}{T}\sum_{t=1}^T q_t$ and $\bar \lambda := \frac{1}{T}\sum_{t=1}^T \lambda_t$ be the average plays of the $q$-player and the $\lambda$-player respectively. Then, we have $L_T(\bar q, \bar \lambda) \le L_T(q, \bar \lambda) + \nu \text{ and }	L_T(\bar q, \bar \lambda) \ge L_T(\bar q, \lambda) - \nu$, under the following conditions,
	\begin{compactitem}
		\item If $\eta = \mathcal{O}(\sqrt{\frac{B^2T}{L^2 (\log T+ 1)}})$, $\nu = \mathcal{O}(\sqrt{\frac{B^2 L^2 (\log T + 1)}{T}})$; $\forall |\lambda| \le B$,  $\forall q \in \Delta$
		\item If $\eta = \frac{1}{T}$, $\nu = \mathcal{O}(\frac{L^2(\log T + 1)^2}{T})$; $\forall \lambda \in \R$, $\forall q \in \Delta$
	\end{compactitem}
\end{theorem}
\begin{proof}
	Recall the definition of $L_T(q, \lambda)$ from the paper, 
	\begin{align}
		L_T(q, \lambda) = \big(\sum_i q_i e_{h_i}\big) + \lambda \big(\sum_i q_i d_{h_i} \big) - \frac{1}{2\eta}\big(\lambda-\lambda_T\big)^2
	\end{align}
	For the sake of this proof, let us define $\zeta_T$ in the following way,
	\begin{align}
		\label{suppeq:zeta_def}
		\zeta_T(\lambda) = \frac{1}{2\eta}\sum_{t=1}^{T}\Big( (\lambda - \lambda_t)^2 - (\lambda - \lambda_T)^2 + (\lambda_t - \lambda_T)^2 \Big)
	\end{align}
	Recollect from \eqref{suppeq:paper_lemma} that $\xi_T = \eta L^2 (\log T + 1)$. We \textbf{outline} the proof as follows,
	\begin{enumerate}
		\item \label{suppitem:ub} First, we compute an upper bound on $L_T(\bar q, \bar \lambda)$,
		\begin{empheq}[box={\Garybox[Average Play Upper Bound]}]{align}
			\label{suppeq:avgplay_ub}
			L_T(\bar q, \bar \lambda) &\le L_T(q, \bar \lambda ) + \frac{\zeta_T(\bar \lambda)}{T} + \frac{\xi_T}{T} \quad \forall q \in \Delta\\
			\text{Also, } L_T(\bar q, \lambda) &\le L_T(q, \bar \lambda ) + \frac{\zeta_T(\lambda)}{T} + \frac{\xi_T}{T} \quad \forall \lambda \in \R ,\forall q \in \Delta
		\end{empheq}
		\item \label{suppitem:lb} Next, we determine an lower bound on $L_T(\bar q, \bar \lambda)$,
		\begin{empheq}[box={\Garybox[Average Play Lower Bound]}]{align}
		\label{suppeq:avgplay_lb}
		L_T(\bar q, \bar \lambda) \ge L_T(\bar q, \lambda ) - \frac{\zeta_T(\lambda)}{T} - \frac{\xi_T}{T} \quad \forall \lambda \in \R
		\end{empheq}
		\item \label{suppitem:lamb} We bound $\frac{\zeta_T(\lambda)}{T} + \frac{\xi_T}{T}$ for the case $|\lambda| \le B$ and show that a $\nu-$approximate saddle point is attained.
		\item \label{suppitem:lamr} We bound $\frac{\zeta_T(\lambda)}{T} + \frac{\xi_T}{T}$ for the case $\lambda \in \R$ and, again, show that $\nu-$approximate saddle point is attained.
	\end{enumerate}
	We write the proofs of the above fours parts one-by-one. Steps~\ref{suppitem:ub},\ref{suppitem:lb} in the above outline are intermediary results used to prove our main results in Steps~\ref{suppitem:lamb},\ref{suppitem:lamr}. Reader can directly move to Steps~\ref{suppitem:lamb},\ref{suppitem:lamr} to see the main proof.\\\\
	{\bf 1. Proof for the result on \textit{Average play Upper Bound}}
	\begin{align}
		L_T(q, \bar \lambda) &= {\textstyle \sum}_i q_i e_{h_i} + \Big(\frac{{\textstyle \sum}_t \lambda_t}{T}\Big)\Big({\textstyle \sum}_i q_i d_{h_i} \Big) - \frac{1}{2\eta}\Big(\frac{{\textstyle \sum}_t \lambda_t}{T} - \lambda_T\Big)^2\\
		\intertext{Exploiting convexity of $\frac{1}{2\eta}\Big(\frac{{\textstyle \sum}_t \lambda_t}{T} - \lambda_T\Big)^2$ via Jenson's Inequality,}
		&\ge \frac{1}{T} \sum_t \Big({\textstyle \sum}_i q_i e_{h_i} + \lambda_t {\textstyle \sum}_i q_i d_{h_i} - \frac{1}{2\eta}(\lambda_t - \lambda_T)^2 \Big)\\
		\intertext{As $h_t = \argmin_{q} L_T(q, \lambda_t)$, we have $L_T(q, \lambda_t) \ge L_T(h_t, \lambda_t)$, hence,}
		&\ge \frac{1}{T} \sum_t \Big( e_{h_t}  + \lambda_t d_{h_t} - \frac{1}{2\eta} (\lambda_t - \lambda_T)^2 \Big)\\
		\intertext{Using the \textit{Cumulative Reward Bound} \eqref{suppeq:paper_lemma},}
		&\ge \frac{{\textstyle \sum}_t e_{h_t}}{T} + \frac{\lambda{\textstyle \sum}_t d_{h_t}}{T} - \frac{1}{T} \sum_t \Big( \frac{(\lambda- \lambda_t)^2}{2\eta} + \frac{(\lambda_t- \lambda_T)^2}{2\eta} \Big)- \frac{\xi_T}{T}\\
		\intertext{Add and subtract $\frac{1}{T}\sum_{t=1}^T\frac{1}{2\eta}(\lambda - \lambda_T)^2$, use $\zeta_T$ from \eqref{suppeq:zeta_def} and regroup the terms,}
		&= ({\textstyle \sum}_i \bar q_i e_{h_i}) + (\lambda{\textstyle \sum}_i \bar q_i d_{h_i}) - \frac{1}{2\eta}(\lambda-\lambda_T)^2 - \frac{\zeta_T(\lambda)}{T} - \frac{\xi_T}{T}\\
		\label{suppeq:final_backward}
		&= L_T(\bar q, \lambda) - \frac{\zeta_T(\lambda)}{T} - \frac{\xi_T}{T}
	\end{align}\\	
		{\bf 2. Proof for the result on \textit{Average play Lower Bound}}
		Proof is similar to Step~\ref{suppitem:ub} so we skip the details. The proof involves finding a lower bound for $L_T(\bar q, \lambda)$ using the \textit{Cumulative Reward Bound} \eqref{suppeq:paper_lemma}. With simple algebraic manipulations and exploiting the convexity of $L_T(\bar q, \lambda)$ via the Jenson's inequality, we obtain the  bound that we state.\\
	{\bf 3. Proof for the case $\bm{ |\lambda| \le B}$}\\
	For the case $|\lambda| \le B$, we have $\zeta_T(\lambda) \le \frac{B^2 T}{\eta}$, which gives,
	\begin{align}
		\label{eqsup:caselamb}
		\frac{\zeta_T(\lambda)}{T} + \frac{\xi_T}{T} &\le \frac{B^2}{\eta} + \frac{\eta L^2 (\log T + 1)}{T} 
	\end{align}
	Minimizing R.H.S in \eqref{eqsup:caselamb} over $\eta$ gives us a $\nu-$ approximate saddle point, 
	\begin{empheq}[box={\Garybox[$\nu-$ approximate saddle point for $|\lambda| \le B$ ]}]{align}
		L_T(\bar q, \bar \lambda) \le L_T(q, \bar \lambda) + \nu \quad &\text{ and }\quad L_T(\bar q, \bar \lambda) \ge L_T(\bar q, \lambda) - \nu\\
		\text{ where } \nu = 2\sqrt{\frac{B^2 L^2 (\log T + 1)}{T}} \quad &\text{ and } \eta = \sqrt{\frac{B^2T}{L^2 (\log T+ 1)}}
	\end{empheq}\\\\
	{\bf 4. Proof for the case $\bm{ \lambda \in \R}$}\\
	We begin the proof by bounding $\frac{\zeta_T(\lambda)}{T} + \frac{\xi_T}{T}$. Let $\lambda_* = \argmax_\lambda L_T(\bar q, \lambda)$. We have a closed form for $\lambda_*$ given by $\lambda_* = \lambda_T + \eta {\textstyle \sum}_i \bar q_i d_{h_i}$. Substituting $\lambda_*$ in $\zeta_T$ gives,
	\begin{align}
		\frac{\zeta_T(\lambda_*)}{T} + \frac{\xi_T}{T}
		\label{eqsup:lasmin}
		&= \frac{1}{2\eta}\frac{1}{T}\sum_t \Big(2(\lambda_t-\lambda_T)^2 + 2\eta (\lambda_T-\lambda_t)({\textstyle \sum}_i \bar q_i d_{h_i}) \Big) + \frac{\xi_T}{T}\\
		\intertext{Recollect that $\lambda_{t+1}-\lambda_t = \frac{\eta}{t} d_{h_t}$ (from~\eqref{step:lambda_relation}). Using telescopic sum on $\lambda_t$, we get $(\lambda_T-\lambda_t) \le \eta L (\log T + 1)$ and $(\lambda_T-\lambda_t)^2 \le \eta^2 L^2 (\log T + 1)^2$. We substitute these in the previous equation \eqref{eqsup:lasmin},}
		\frac{\zeta_T(\lambda_*)}{T} + \frac{\xi_T}{T} &\le \eta L^2 (\log T + 1)^2 + \eta L^2 (\log T + 1) + \frac{\eta L^2 (\log T + 1)}{T}\\
		\intertext{Setting $\eta=\frac{1}{T}$, we get}
		\label{suppeq:zeta_bound}
		\frac{\zeta_T(\lambda_*)}{T} + \frac{\xi_T}{T} &\le \mathcal{O}(\frac{L^2(\log T + 1)^2}{T}) := \nu
	\end{align}
	Using ~\eqref{suppeq:zeta_bound}, we prove the convergence of $\lambda$ in the following way,
	\begin{align}
		L_T(\bar q, \lambda) &\le L_T(\bar q, \lambda_*) \quad \text{\Big(as $\lambda_*$ is the maximizer of $L_T(\bar q, \lambda)$\Big)} \\
		&\le L_T(\bar q, \bar \lambda) + \frac{\zeta_T(\lambda_*)}{T} + \frac{\xi_T}{T} \quad \text{\Big(\textit{Average Play Lower Bound} \Big)}\\
		&\le L_T(\bar q, \bar \lambda) + \nu \quad \text{\Big(from ~\eqref{suppeq:zeta_bound}\Big)}
	\end{align}
	We prove the convergence of $q$ in the following way. For any $\lambda \in \R$,
	\begin{align}
		L_T(q, \bar \lambda) &\ge L_T(\bar q, \lambda_*) - \frac{\zeta_T(\lambda_*)}{T} - \frac{\xi_T}{T} \hskip 7.9pt \text{\Big(\textit{Average Play Upper Bound} \eqref{suppeq:final_backward}\Big)}\\
		&\ge L_T(\bar q, \lambda_*) - \nu \quad \text{\Big(from \eqref{suppeq:zeta_bound}\Big)}\\
		&\ge L_T(\bar q, \bar \lambda) - \nu \quad \text{\Big(as $\lambda_*$ is the maximizer of $L_T(\bar q, \lambda)$\Big)}
	\end{align}
	Therefore,
	\begin{empheq}[box={\Garybox[$\nu-$ approximate saddle point for $\lambda \in \R$]}]{align}
		L_T(\bar q, \bar \lambda) \le L_T(q, \lambda) + \nu \quad &\text{ and }\quad L_T(\bar q, \bar \lambda) \ge L_T(\bar q, \lambda) - \nu\\
		\text{ where } \nu = \mathcal{O}\bigg(\frac{L^2(\log T + 1)^2}{T}\bigg) \quad &\text{ and } \quad \eta = \frac{1}{T}
	\end{empheq}
\end{proof}

\subsection{More details on \textit{FairALM: DeepNet Classifier}}
Recall that in $\S~5.2$ in the paper, we identified a key difficulty when extending our algorithm to
deep networks. The main issue is that the set of classifiers $|\mathcal{H}|$ is
not a finite set. We argued that leveraging stochastic gradient descent (SGD)
on an over-parameterized network eliminates this issue. When using SGD, few additional modifications of Alg~$1$ (in the paper) are helpful, such as replacing the non-differentiable indicator function $\mathbbm{1}[\cdot]$ with a smooth surrogate function and computing the empirical estimates of the errors and conditional means denoted by $\hat e_{h}(z)/\hat \mu_{h}^{s}(z)$ respectively. These changes modify our objective to a form that is not a zero-sum game,
\begin{align}
	\label{eqsup:nzobj}
	\max_{\lambda} \min_w \Big( \hat e_{h_w} + \lambda (\hat \mu_{h_w}^{s_0} - \hat \mu_{h_w}^{s_1}) -\frac{1}{2\eta} (\lambda - \lambda_t)^2 \Big)
\end{align}
We use DP constraint in \eqref{eqsup:nzobj}, other fairness metrics discussed in the paper are valid as well. A closed-form solution for $\lambda$  can be achieved by solving an upper bound to \eqref{eqsup:nzobj} obtained by exchanging the ``max"/``min" operations.
\begin{align}
\max_{\lambda} \min_w \Big( \hat e_{h_w} &+ \lambda (\hat \mu_{h_w}^{s_0} - \hat \mu_{h_w}^{s_1}) -\frac{1}{2\eta} (\lambda - \lambda_t)^2 \Big) \\
\label{eqsup:lagub}
&\le \min_w \max_{\lambda} \Big( \hat e_{h_w} + \lambda (\hat \mu_{h_w}^{s_0} - \hat \mu_{h_w}^{s_1}) -\frac{1}{2\eta} (\lambda - \lambda_t)^2 \Big)
\intertext{Substituting the closed form solution $\lambda = \lambda_t + \eta(\hat \mu_{h_w}^{s_0} - \hat \mu_{h_w}^{s_1})$ in \eqref{eqsup:lagub},}
\label{eqsup:convupperhalf}
&\le \min_w \Big( \hat e_{h_w} + + \lambda_t (\hat \mu_{h_w}^{s_0} - \hat \mu_{h_w}^{s_1}) + \frac{\eta}{2}  (\hat \mu_{h_w}^{s_0} - \hat \mu_{h_w}^{s_1})^2  \Big)
\intertext{Note that the surrogate function defined within $\hat \mu_{h_w}^{s}$ is convex and non-negative, hence, we can exploit Jenson's inequality to eliminate the power~$2$ in \eqref{eqsup:convupperhalf} to give us a convenient upper bound,}
\label{eqsup:convupper}
&\le \min_w \Big( \hat e_{h_w} + (\lambda_t + \eta) \hat \mu_{h_w}^{s_0} - (\lambda_t - \eta) \hat \mu_{h_w}^{s_1}\Big)
\end{align}
In order to obtain a good minima in \eqref{eqsup:convupper}, it may be essential to run the SGD on \eqref{eqsup:convupper} a few times: for ImSitu experiments, SGD was run on \eqref{eqsup:convupper} for $5$ times. We also gradually increase the parameter $\eta$ with time as $\eta_{t} = \eta_{t-1} (1 + \eta_\beta)$ for a small non-negative value for $\eta_\beta$, e.g., $\eta_\beta\approx 0.01$. This is a common practice in augmented Lagrangian methods, see
\cite{bertsekas2014constrained} (page $104$). The overall algorithm is available in the paper as Alg.~$2$. The key primal and dual steps can be seen in the following section.

\clearpage

\subsection{Algorithm for baselines}
We provide the primal and dual steps used for the baseline algorithms for the ImSitu experiments from the paper. The basic framework for all the baselines remains the same as Alg.~$2$ in the paper. For Proxy-Lagrangian, only the key ideas in \cite{cotter2018two} were adopted for implementation.
\begin{empheq}[box={\GaryboxAlg[Unconstrained]}]{align*}
	\text{PRIMAL:} &\quad v_t \in \partial \hat e_{h_w} \\
	\text{DUAL:} &\quad \text{None}	
\end{empheq}
\begin{empheq}[box={\GaryboxAlg[$\ell_2$ Penalty]}]{align*}
\text{PRIMAL:} &\quad v_t \in \partial \Big( \hat e_{h_w} + \eta (\hat \mu_{h_w}^{s_0} - \hat \mu_{h_w}^{s_1})^2 \Big) \\
\text{DUAL:} &\quad \text{None} \\	
\text{Parameters:} &\quad \text{Penalty Parameter } \eta
\end{empheq}
\begin{empheq}[box={\GaryboxAlg[Reweight]}]{align*}
\text{PRIMAL:} &\quad v_t \in \partial \Big( \hat e_{h_w} + \eta_0 \hat \mu_{h_w}^{s_0} + \eta_1 \hat \mu_{h_w}^{s_1} \Big) \\
\text{DUAL:} &\quad \text{None} \\	
\text{Parameters:} &\quad \eta_i \propto 1 / (\# \text{ samples in } s_i)
\end{empheq}
\begin{empheq}[box={\GaryboxAlg[Lagrangian \cite{zhao2017men}]}]{align*}
\text{PRIMAL:} &\quad v_t \in \partial \Big( \hat e_{h_w} + \lambda^{0\backslash 1}_t (\hat \mu_{h_w}^{s_0} - \hat \mu_{h_w}^{s_1} - \epsilon) +  \lambda^{1\backslash 0}_t (\hat \mu_{h_w}^{s_1} - \hat \mu_{h_w}^{s_0} - \epsilon)\Big) \\
\text{DUAL:} &\quad \lambda^{i\backslash j}_{t+1} \leftarrow \max \big(0, \lambda^{i\backslash j}_t + \eta_{i\backslash j} (\hat \mu_{h_w}^{s_i} - \hat \mu_{h_w}^{s_j} - \epsilon) \big)\\	
\text{Parameters:} &\quad \text{Dual step sizes } \eta_{0 \backslash 1}, \eta_{1 \backslash 0} \text{ Tol. } \epsilon \approx 0.05. \hskip 5pt \mtiny{i\backslash j \in \{ 0\backslash 1, 1 \backslash 0\} }
\end{empheq}
\begin{empheq}[box={\GaryboxAlg[Proxy-Lagrangian \cite{cotter2018two}]}]{align*}
\text{PRIMAL:} &\quad v_t \in \partial \Big( \hat e_{h_w} + \lambda^{0 \backslash 1}_t (\hat \mu_{h_w}^{s_0} - \hat \mu_{h_w}^{s_1} - \epsilon) +  \lambda^{1 \backslash 0}_t (\hat \mu_{h_w}^{s_1} - \hat \mu_{h_w}^{s_0} - \epsilon)\Big) \\
\text{DUAL:} &\quad \theta^{i \backslash j}_{t+1} \leftarrow \theta^{i\backslash j}_{t} + \eta_{i\backslash j} (\hat \mu_{h_w}^{s_i} - \hat \mu_{h_w}^{s_j} - \epsilon) \\
&\quad \lambda^{i\backslash j}_{t+1} \leftarrow B \frac{\exp{\theta^{i\backslash j}_{t+1}}}{1 + \exp{\theta^{i\backslash j}_{t+1}} + \exp{\theta^{j\backslash i}_{t+1}}}\\	
\text{Parameters:} &\quad \text{Dual step sizes } \eta_{0\backslash 1}/\eta_{1\backslash 0}. \text{ Tol. } \epsilon \approx 0.05, \text{ Hyperparam. } B \\
&\quad \text{No surrogates in DUAL for } \hat \mu^{s_0}_{h_w} / \hat \mu^{s_1}_{h_w}. \hskip 5pt \mtiny{i\backslash j \in \{ 0\backslash 1, 1 \backslash 0\} }
\end{empheq}
\begin{empheq}[box={\GaryboxAlg[FairALM]}]{align*}
\text{PRIMAL:} &\quad v_t \in \partial \Big(\hat e_{h_w}(z) + (\lambda_t + \eta) \hat \mu_{h_w}^{s_0} (z) - (\lambda_t-\eta) \hat \mu_{h_w}^{s_1} (z) \Big)  \\
\text{DUAL:} &\quad \lambda_{t+1} \leftarrow \lambda_t + \eta \big(\hat \mu_{h_{w}}^{s_0} - \hat \mu_{h_{w}}^{s_1} \big)  \\	
\text{Parameters:} &\quad \text{Dual Step Size } \eta
\end{empheq}


\clearpage
\subsection{Supplementary Results on CelebA}
{\bf Additional Results. } The dual step size $\eta$ is a key parameter in {\rm FairALM} training. Analogous to the dual step size $\eta$ we have the penalty parameter in $\ell_2$ penalty training, also denoted by $\eta$.  It can be seen from Figure~\ref{fig:fairalm_ablation_celeba} and Figure~\ref{fig:l2penalty_ablation_celeba} that {\rm FairALM} is more robust to different choices of $\eta$ than $\ell_2$ penalty. The target class in this section is \textit{attractiveness} and protected attribute is \textit{gender}.
\begin{figure*}[!h]
  	\centering
  	\frame{
	\begin{minipage}{0.2\linewidth}
		\includegraphics[width=\linewidth]{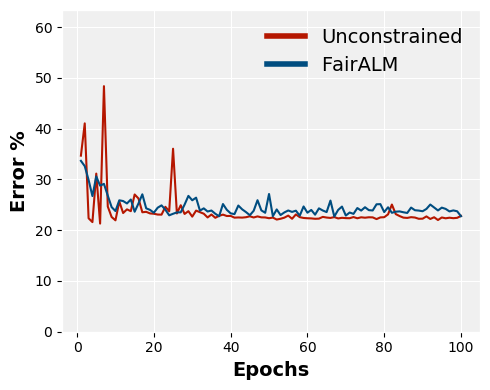}
	\end{minipage}
	$(\eta = 20)$
	\begin{minipage}{0.2\linewidth}
		\includegraphics[width=\linewidth]{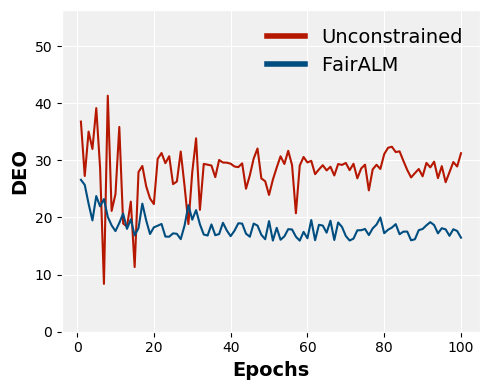}
	\end{minipage}}
	\hskip 4pt
	\frame{
	\begin{minipage}{0.2\linewidth}
		\includegraphics[width=\linewidth]{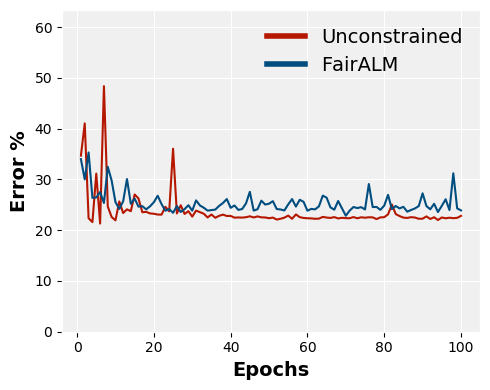}
	\end{minipage}
	$(\eta = 40)$
	\begin{minipage}{0.2\linewidth}
		\includegraphics[width=\linewidth]{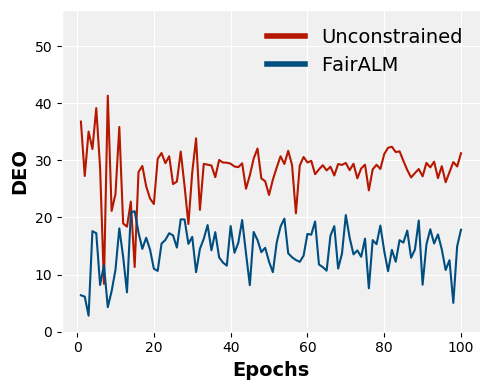}
	\end{minipage}} %
	\vskip 4pt
	\frame{
	\begin{minipage}{0.2\linewidth}
		\includegraphics[width=\linewidth]{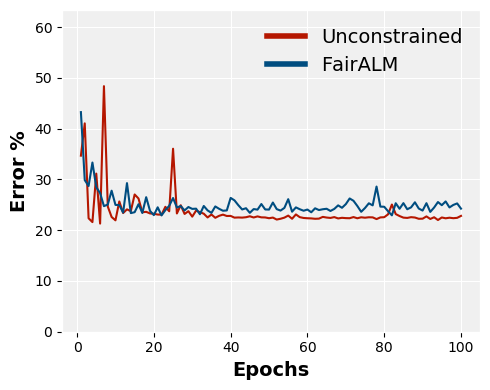}
	\end{minipage}
	$(\eta = 60)$
	\begin{minipage}{0.2\linewidth}
		\includegraphics[width=\linewidth]{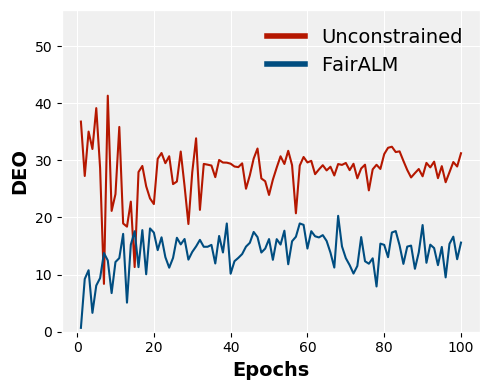}
	\end{minipage}}
	\hskip 4pt
	\frame{
	\begin{minipage}{0.2\linewidth}
		\includegraphics[width=\linewidth]{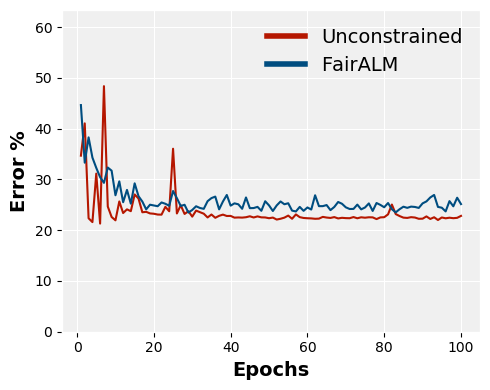}
	\end{minipage}
	$(\eta = 80)$
	\begin{minipage}{0.2\linewidth}
		\includegraphics[width=\linewidth]{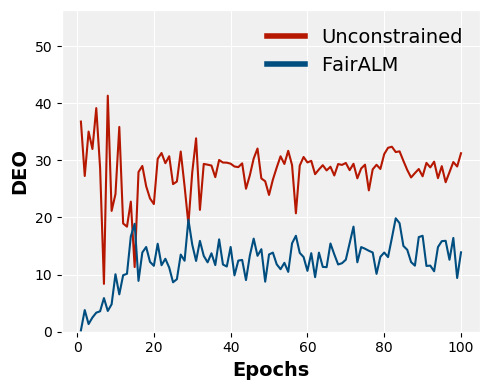}
	\end{minipage}}%
	\caption{\label{fig:fairalm_ablation_celeba}\footnotesize \textbf{FairALM Ablation on CelebA. } For a given $\eta$, the left image represents the test error and the right image shows the DEO measure. We study the effect of varying the dual step size $\eta$ on FairALM. We observe that the performance of FairALM is consistent over a wide range of $\eta$ values.}
\end{figure*}

\begin{figure*}[!h]
	\centering
	\frame{
		\begin{minipage}{0.19\linewidth}
			\includegraphics[width=\linewidth]{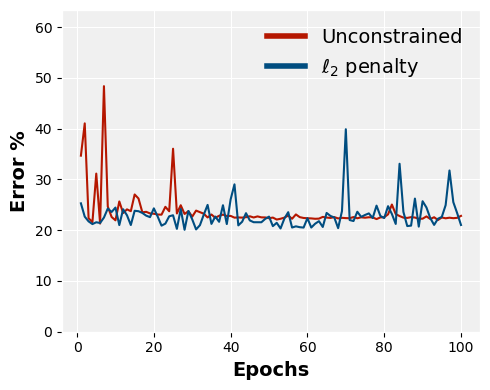}
		\end{minipage}
		$(\eta = 0.001)$
		\begin{minipage}{0.19\linewidth}
			\includegraphics[width=\linewidth]{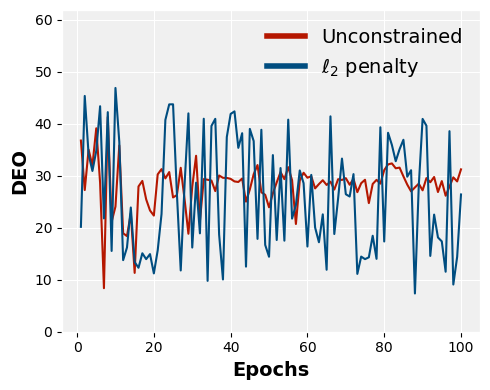}
	\end{minipage}} 
	\hskip 4pt
	\frame{
		\begin{minipage}{0.19\linewidth}
			\includegraphics[width=\linewidth]{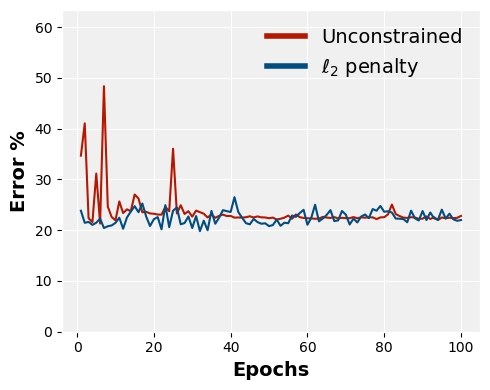}
		\end{minipage}
		$(\eta = 0.01)$ \hskip 4pt
		\begin{minipage}{0.19\linewidth}
			\includegraphics[width=\linewidth]{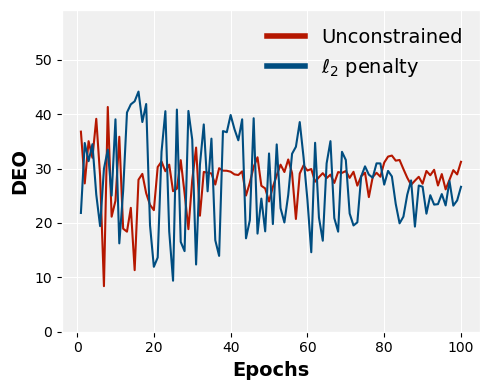}
	\end{minipage}} %
	\vskip 4pt
	\frame{
		\begin{minipage}{0.19\linewidth}
			\includegraphics[width=\linewidth]{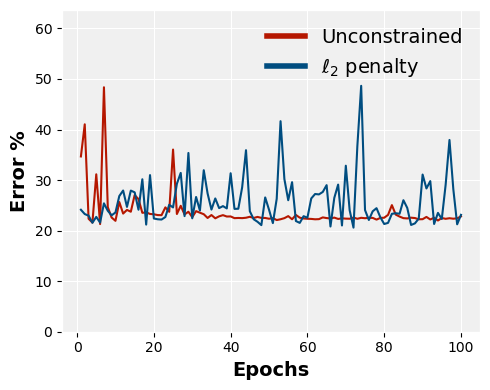}
		\end{minipage}
		$(\eta = 0.1)$ \hskip 8pt
		\begin{minipage}{0.19\linewidth}
			\includegraphics[width=\linewidth]{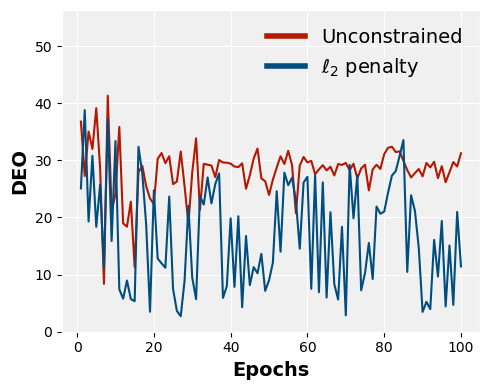}
	\end{minipage}} 
	\hskip 4pt
	\frame{
		\begin{minipage}{0.19\linewidth}
			\includegraphics[width=\linewidth]{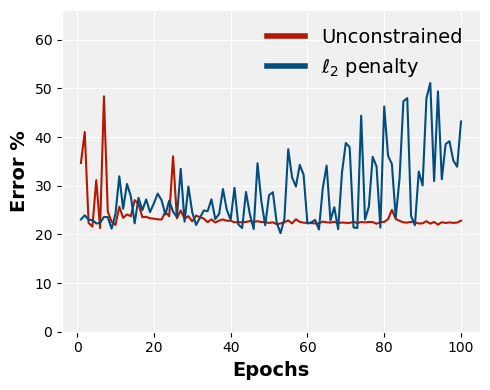}
		\end{minipage}
		$(\eta = 1)$ \hskip 12pt
		\begin{minipage}{0.19\linewidth}
			\includegraphics[width=\linewidth]{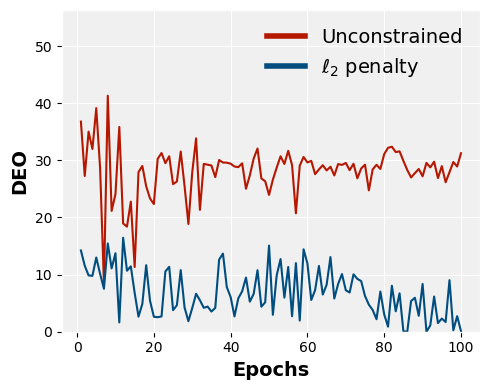}
	\end{minipage}}%
	\caption{\label{fig:l2penalty_ablation_celeba} \footnotesize \textbf{$\bm{\ell_2}$ Penalty Ablation on CelebA} For each $\eta$ value, the left image represents the test set errors and the right image shows the fairness measure (DEO). We investigate a popular baseline to impose fairness constraint which is the $\ell_2$ penalty. We study the effect of varying the penalty parameter $\eta$ in this figure. We observe that training with $\ell_2$ penalty is quite unstable. For $\eta > 1$, the algorithm doesn't converge and raises numerical errors.}
\end{figure*}

\clearpage
{\bf More Interpretability Results. }
We present the activation maps obtained when running the {\em FairALM} algorithm, unconstrained algorithm and the gender classification task. We show our results in Figure~\ref{more_acti}. The target class is \textit{attractiveness} and protected attribute is gender. We threshold the maps to show only the most significant colors. The maps from gender classification task look at gender-revealing attributes such as presence of \textit{long-hair}. The unconstrained model looks mostly at the entire image. {\em FairALM} looks at only a specific region of the face which is not gender revealing.  
\begin{figure*}[!h]
	\centering
	{\setlength{\fboxsep}{4pt}\fbox{
		\begin{minipage}{0.12\linewidth}
			\figuretitle{\colorbox{mylightgrey}{Gender}}
			\adjustbox{cfbox=mygrey 2pt 0pt}{\includegraphics[width=\linewidth]{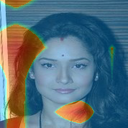}}
		\end{minipage}
		\begin{minipage}{0.12\linewidth}
			\figuretitle{\colorbox{mylightpink}{Unconstrained}}
			\adjustbox{cfbox=mypink 2pt 0pt}{\includegraphics[width=\linewidth]{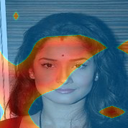}}
		\end{minipage}
		\begin{minipage}{0.12\linewidth}
			\figuretitle{\colorbox{mylightgreen}{FairALM}}
			\adjustbox{cfbox=mygreen 2pt 0pt}{\includegraphics[width=\linewidth]{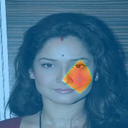}}
		\end{minipage}}}
		\quad
		{\setlength{\fboxsep}{4pt}\fbox{
			\begin{minipage}{0.12\linewidth}
				\figuretitle{\colorbox{mylightgrey}{Gender}}
				\adjustbox{cfbox=mygrey 2pt 0pt}{\includegraphics[width=\linewidth]{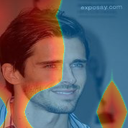}}
			\end{minipage}
			\begin{minipage}{0.12\linewidth}
				\figuretitle{\colorbox{mylightpink}{Unconstrained}}
				\adjustbox{cfbox=mypink 2pt 0pt}{\includegraphics[width=\linewidth]{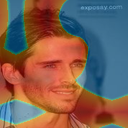}}
			\end{minipage}
			\begin{minipage}{0.12\linewidth}
				\figuretitle{\colorbox{mylightgreen}{FairALM}}
				\adjustbox{cfbox=mygreen 2pt 0pt}{\includegraphics[width=\linewidth]{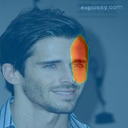}}
	\end{minipage}}}

	\vskip 4pt
	{\setlength{\fboxsep}{4pt}\fbox{
		\begin{minipage}{0.12\linewidth}
			\figuretitle{\colorbox{mylightgrey}{Gender}}
			\adjustbox{cfbox=mygrey 2pt 0pt}{\includegraphics[width=\linewidth]{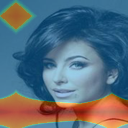}}
		\end{minipage}
		\begin{minipage}{0.12\linewidth}
			\figuretitle{\colorbox{mylightpink}{Unconstrained}}
			\adjustbox{cfbox=mypink 2pt 0pt}{\includegraphics[width=\linewidth]{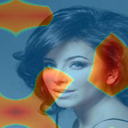}}
		\end{minipage}
		\begin{minipage}{0.12\linewidth}
			\figuretitle{\colorbox{mylightgreen}{FairALM}}
			\adjustbox{cfbox=mygreen 2pt 0pt}{\includegraphics[width=\linewidth]{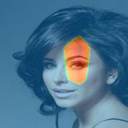}}
\end{minipage}}}
\quad
{\setlength{\fboxsep}{4pt}\fbox{
		\begin{minipage}{0.12\linewidth}
			\figuretitle{\colorbox{mylightgrey}{Gender}}
			\adjustbox{cfbox=mygrey 2pt 0pt}{\includegraphics[width=\linewidth]{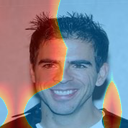}}
		\end{minipage}
		\begin{minipage}{0.12\linewidth}
			\figuretitle{\colorbox{mylightpink}{Unconstrained}}
			\adjustbox{cfbox=mypink 2pt 0pt}{\includegraphics[width=\linewidth]{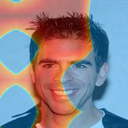}}
		\end{minipage}
		\begin{minipage}{0.12\linewidth}
			\figuretitle{\colorbox{mylightgreen}{FairALM}}
			\adjustbox{cfbox=mygreen 2pt 0pt}{\includegraphics[width=\linewidth]{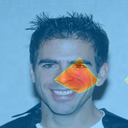}}
\end{minipage}}}

	\vskip 4pt
{\setlength{\fboxsep}{4pt}\fbox{
		\begin{minipage}{0.12\linewidth}
			\figuretitle{\colorbox{mylightgrey}{Gender}}
			\adjustbox{cfbox=mygrey 2pt 0pt}{\includegraphics[width=\linewidth]{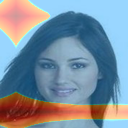}}
		\end{minipage}
		\begin{minipage}{0.12\linewidth}
			\figuretitle{\colorbox{mylightpink}{Unconstrained}}
			\adjustbox{cfbox=mypink 2pt 0pt}{\includegraphics[width=\linewidth]{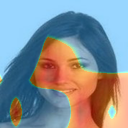}}
		\end{minipage}
		\begin{minipage}{0.12\linewidth}
			\figuretitle{\colorbox{mylightgreen}{FairALM}}
			\adjustbox{cfbox=mygreen 2pt 0pt}{\includegraphics[width=\linewidth]{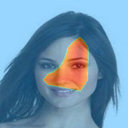}}
\end{minipage}}}
\quad
{\setlength{\fboxsep}{4pt}\fbox{
		\begin{minipage}{0.12\linewidth}
			\figuretitle{\colorbox{mylightgrey}{Gender}}
			\adjustbox{cfbox=mygrey 2pt 0pt}{\includegraphics[width=\linewidth]{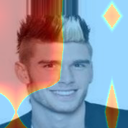}}
		\end{minipage}
		\begin{minipage}{0.12\linewidth}
			\figuretitle{\colorbox{mylightpink}{Unconstrained}}
			\adjustbox{cfbox=mypink 2pt 0pt}{\includegraphics[width=\linewidth]{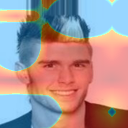}}
		\end{minipage}
		\begin{minipage}{0.12\linewidth}
			\figuretitle{\colorbox{mylightgreen}{FairALM}}
			\adjustbox{cfbox=mygreen 2pt 0pt}{\includegraphics[width=\linewidth]{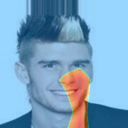}}
\end{minipage}}}

	\vskip 4pt
{\setlength{\fboxsep}{4pt}\fbox{
		\begin{minipage}{0.12\linewidth}
			\figuretitle{\colorbox{mylightgrey}{Gender}}
			\adjustbox{cfbox=mygrey 2pt 0pt}{\includegraphics[width=\linewidth]{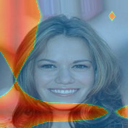}}
		\end{minipage}
		\begin{minipage}{0.12\linewidth}
			\figuretitle{\colorbox{mylightpink}{Unconstrained}}
			\adjustbox{cfbox=mypink 2pt 0pt}{\includegraphics[width=\linewidth]{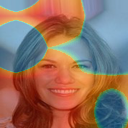}}
		\end{minipage}
		\begin{minipage}{0.12\linewidth}
			\figuretitle{\colorbox{mylightgreen}{FairALM}}
			\adjustbox{cfbox=mygreen 2pt 0pt}{\includegraphics[width=\linewidth]{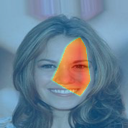}}
\end{minipage}}}
\quad
{\setlength{\fboxsep}{4pt}\fbox{
		\begin{minipage}{0.12\linewidth}
			\figuretitle{\colorbox{mylightgrey}{Gender}}
			\adjustbox{cfbox=mygrey 2pt 0pt}{\includegraphics[width=\linewidth]{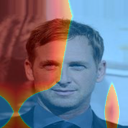}}
		\end{minipage}
		\begin{minipage}{0.12\linewidth}
			\figuretitle{\colorbox{mylightpink}{Unconstrained}}
			\adjustbox{cfbox=mypink 2pt 0pt}{\includegraphics[width=\linewidth]{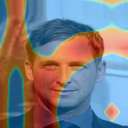}}
		\end{minipage}
		\begin{minipage}{0.12\linewidth}
			\figuretitle{\colorbox{mylightgreen}{FairALM}}
			\adjustbox{cfbox=mygreen 2pt 0pt}{\includegraphics[width=\linewidth]{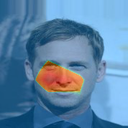}}
\end{minipage}}}

	\vskip 4pt
{\setlength{\fboxsep}{4pt}\fbox{
		\begin{minipage}{0.12\linewidth}
			\figuretitle{\colorbox{mylightgrey}{Gender}}
			\adjustbox{cfbox=mygrey 2pt 0pt}{\includegraphics[width=\linewidth]{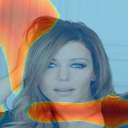}}
		\end{minipage}
		\begin{minipage}{0.12\linewidth}
			\figuretitle{\colorbox{mylightpink}{Unconstrained}}
			\adjustbox{cfbox=mypink 2pt 0pt}{\includegraphics[width=\linewidth]{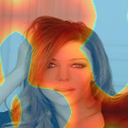}}
		\end{minipage}
		\begin{minipage}{0.12\linewidth}
			\figuretitle{\colorbox{mylightgreen}{FairALM}}
			\adjustbox{cfbox=mygreen 2pt 0pt}{\includegraphics[width=\linewidth]{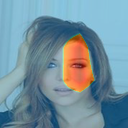}}
\end{minipage}}}
\quad
{\setlength{\fboxsep}{4pt}\fbox{
		\begin{minipage}{0.12\linewidth}
			\figuretitle{\colorbox{mylightgrey}{Gender}}
			\adjustbox{cfbox=mygrey 2pt 0pt}{\includegraphics[width=\linewidth]{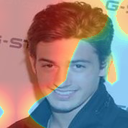}}
		\end{minipage}
		\begin{minipage}{0.12\linewidth}
			\figuretitle{\colorbox{mylightpink}{Unconstrained}}
			\adjustbox{cfbox=mypink 2pt 0pt}{\includegraphics[width=\linewidth]{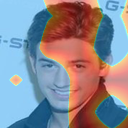}}
		\end{minipage}
		\begin{minipage}{0.12\linewidth}
			\figuretitle{\colorbox{mylightgreen}{FairALM}}
			\adjustbox{cfbox=mygreen 2pt 0pt}{\includegraphics[width=\linewidth]{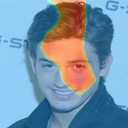}}
\end{minipage}}}
	\caption{\label{more_acti} \footnotesize \textbf{Interpretability in CelebA.} We find that an unconstrained model picks up a lot of gender revealing attributes however FairALM doesn't. The image labelled {\tt Gender} denotes the map of a gender classification task. We observe overlap between the maps of gender classification task and the unconstrained model. The activation maps are regulated to show colors above a fixed threshold to highlight the most significant regions used by a model.}
\end{figure*}

\clearpage
\subsection{Supplementary Results on ImSitu}
{\bf Detailed Setup.} We use the standard ResNet-18 architecture for the base model. We initialize the weights of the conv layers weights from ResNet-18 trained on ImageNet (ILSVRC). We train the model using SGD optimizer and a batch size of $256$. For first few epochs ($\approx20$) only the linear layer is trained  with a learning rate of $0.01/0.005$. Thereafter, the entire model is trained end to end with a lower learning rate of $0.001/0.0005$ till the accuracy plateaus. 

{\bf Meaning of Target class $(+)$. } Target class $(+)$ is something that a classifier tries to predict from an image. Recall the basic notations $\S~2$ from the paper, $\mu_h^{s_i,t_j}:=\mu_h|(s=s_i,t=t_j)$ denotes the elementary conditional expectation of some function $\mu_h$ with respect to two random variables $s,t$. When we say we are imposing DEO for a target class $t_j$ we refer to imposing constraint on the difference in conditional expectation of the two groups of $s$ for the class $t_j$, that is, $d_h = \mu_h^{s_0,t_j} - \mu_h^{s_1,t_j}$. For example, for \textit{Cooking $(+)$} vs \textit{Driving $(-)$} problem when we say \textit{Cooking $(+)$} is regarded as the target class we mean that $t_j=\textit{cooking}$ and hence the DEO constraint is of the form $d_h = \mu_h^{s_0,cooking} - \mu_h^{s_1,cooking}$.

{\bf Supplementary Training Profiles.} We plot the test set errors and the DEO measure during the course of training for the verb pair classifications reported in the paper. We compare against the baselines discussed in Table~$1$ of the paper. The plots in Fig.~\ref{fig:fairalm_supp_plots} below supplement Fig.~$5$ in the paper.
\begin{figure}[!h]
	\centering
	\begin{minipage}{0.8\linewidth}
		\centering
		\frame{\includegraphics[width=0.24\linewidth]{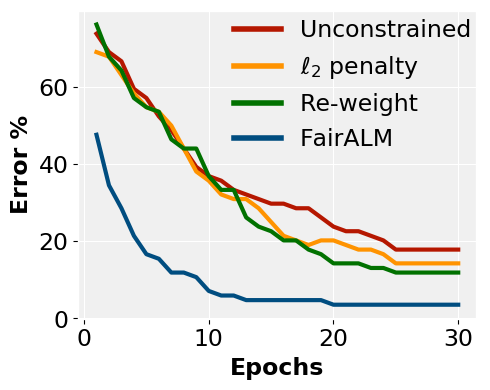}
		\includegraphics[width=0.24\linewidth]{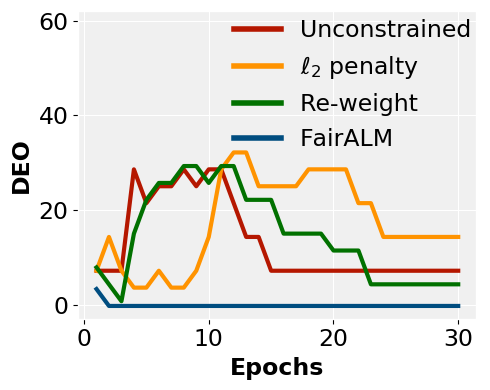}} \hskip 4pt
		\frame{\includegraphics[width=0.24\linewidth]{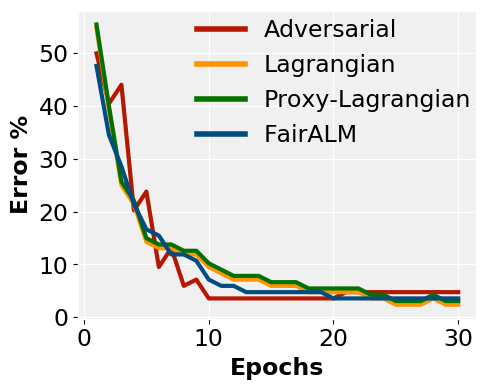}
		\includegraphics[width=0.24\linewidth]{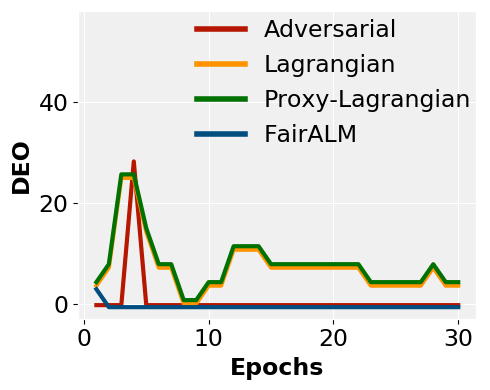}}
		\vskip -1pt
		\par Cooking $\tiny(+)$ Driving $\tiny(-)$
	\end{minipage}
	\vskip 6pt
	\begin{minipage}{0.8\linewidth}
	\centering
	\frame{\includegraphics[width=0.24\linewidth]{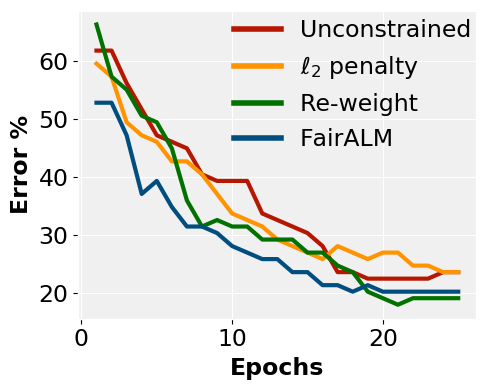}
		\includegraphics[width=0.24\linewidth]{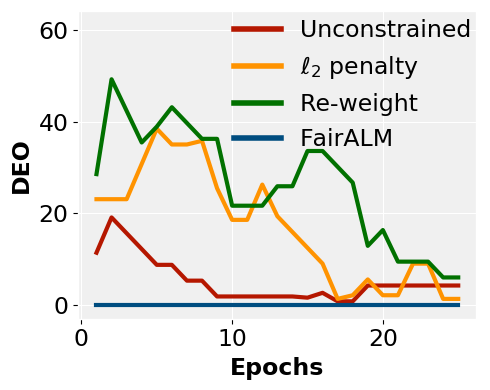}} \hskip 4pt
	\frame{\includegraphics[width=0.24\linewidth]{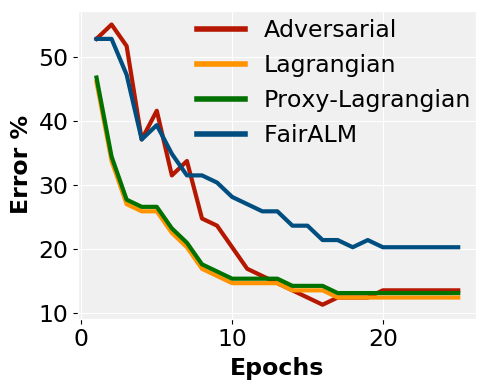}
		\includegraphics[width=0.24\linewidth]{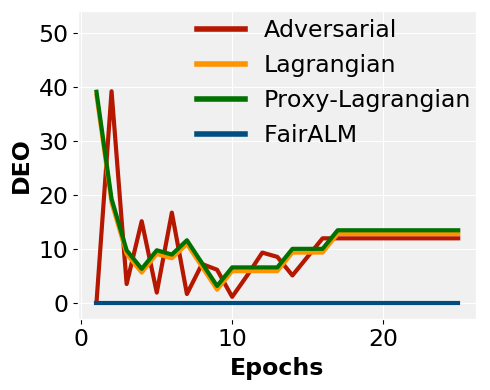}}	
	\vskip -1pt
	\par Shaving $\tiny(+)$ Moisturizing $\tiny(-)$
\end{minipage}	
	\vskip 6pt
\begin{minipage}{0.8\linewidth}
	\centering
	\frame{\includegraphics[width=0.24\linewidth]{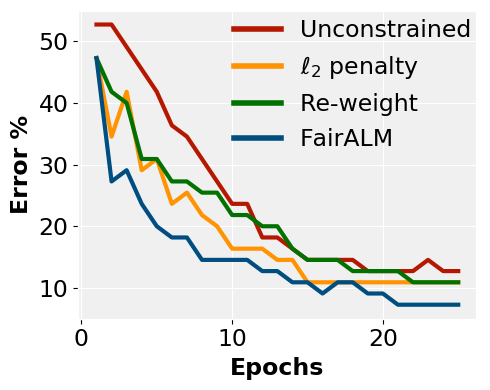}
		\includegraphics[width=0.24\linewidth]{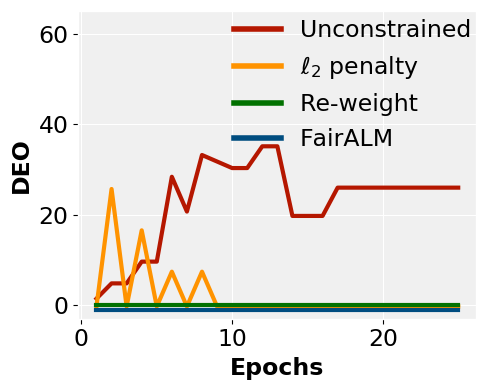}} \hskip 4pt
	\frame{\includegraphics[width=0.24\linewidth]{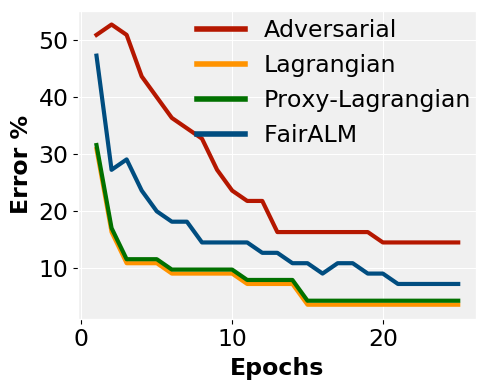}
		\includegraphics[width=0.24\linewidth]{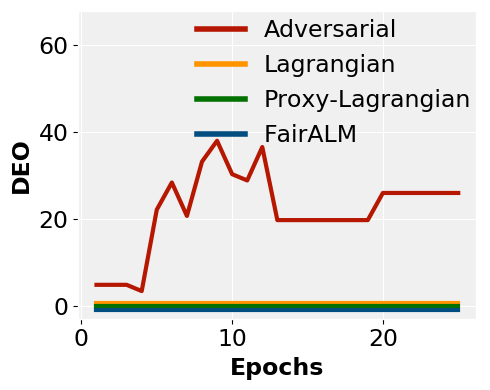}}	
	\vskip -1pt
	\par Washing $\tiny(+)$ Saluting $\tiny(-)$
\end{minipage}	
	\vskip 6pt
\begin{minipage}{0.8\linewidth}
	\centering
	\frame{\includegraphics[width=0.24\linewidth]{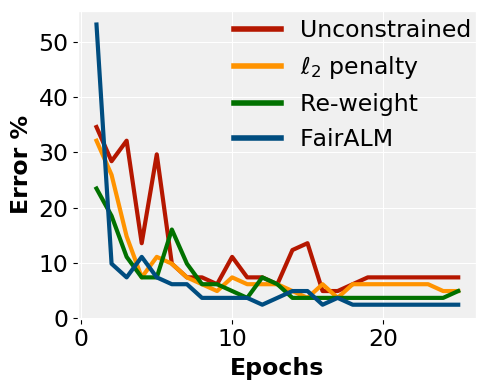}
		\includegraphics[width=0.24\linewidth]{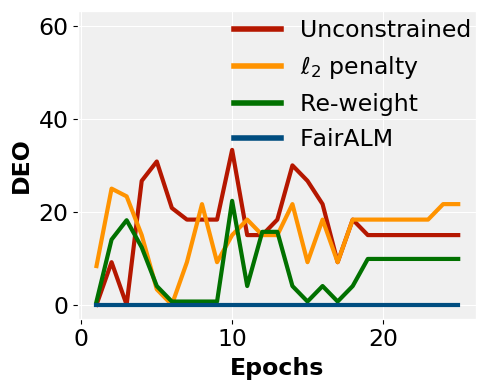}} \hskip 4pt
	\frame{\includegraphics[width=0.24\linewidth]{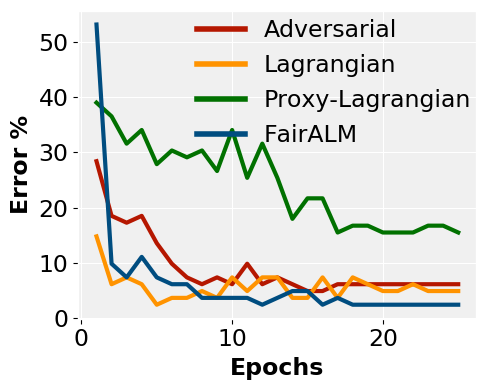}
		\includegraphics[width=0.24\linewidth]{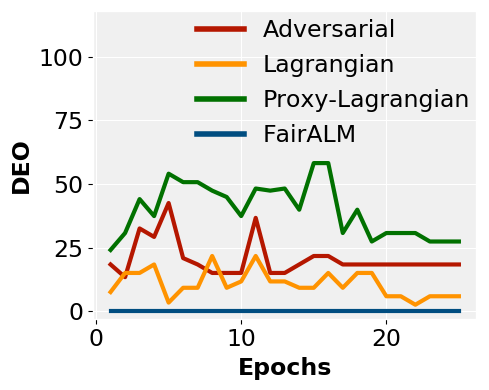}}	
	\vskip -1pt
	\par Assembling $\tiny(+)$ Hanging $\tiny(-)$
\end{minipage}
	\caption{\label{fig:fairalm_supp_plots}\footnotesize \textbf{Supplementary Training Profiles.} FairALM consistently achieves minimum DEO across different verb pair classifications.}
\end{figure}

{\bf Additional qualitative results} We show the activation maps in Fig.~\ref{fig:imsitu-supp-plots} to illustrate that the features used by FairALM model are more aligned with the action/verb present in the image and are not gender leaking. The verb pairs have been chosen randomly from the list provided in \cite{zhao2017men}. In all the cases {\tt Gender} is considered as the protected attribute. The activation maps are regulated to show colors above a fixed threshold in order to highlight the most significant regions used by a model to make a prediction.
\begin{figure*}[!h]
	\centering
	\begin{minipage}{0.9\linewidth}
		\includegraphics[width=0.5\linewidth]{supp_figs/imsitu_eccv/cam_supp_filtered/microwaving_pumping_1.png} \hskip 4pt
		\includegraphics[width=0.5\linewidth]{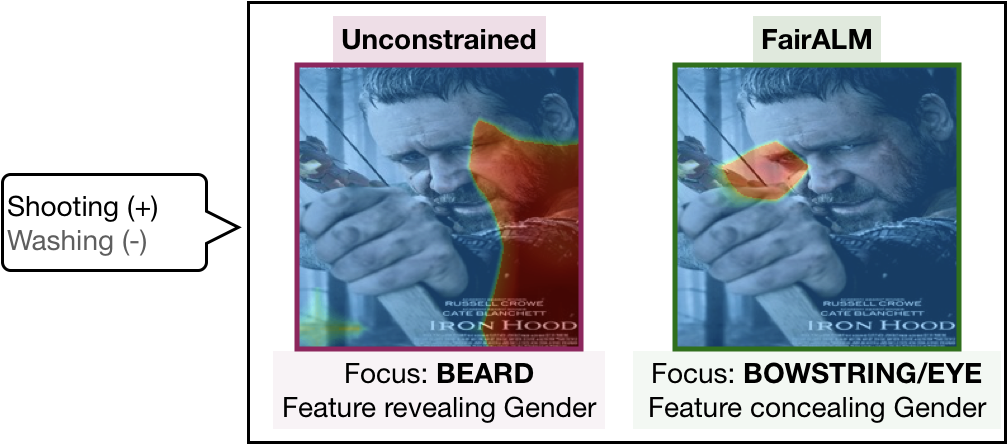}
	\end{minipage}%
	\vskip 4pt
	\begin{minipage}{0.9\linewidth}
		\includegraphics[width=0.5\linewidth]{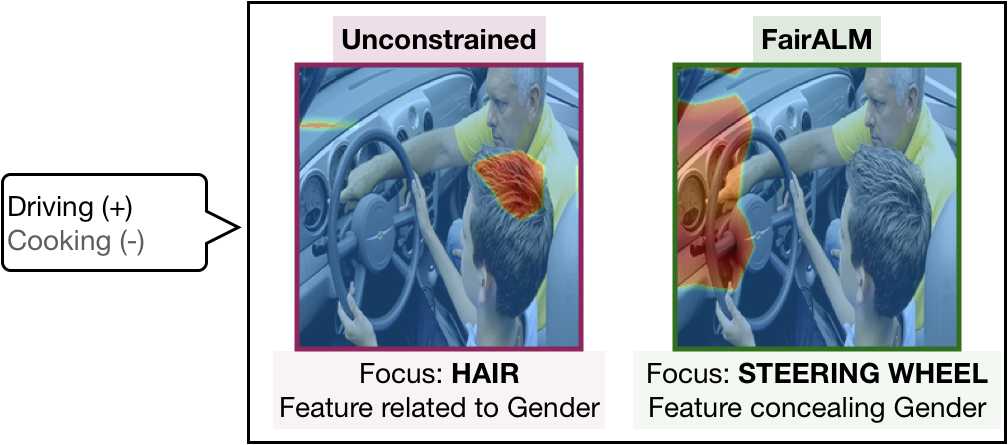} \hskip 4pt
		\includegraphics[width=0.5\linewidth]{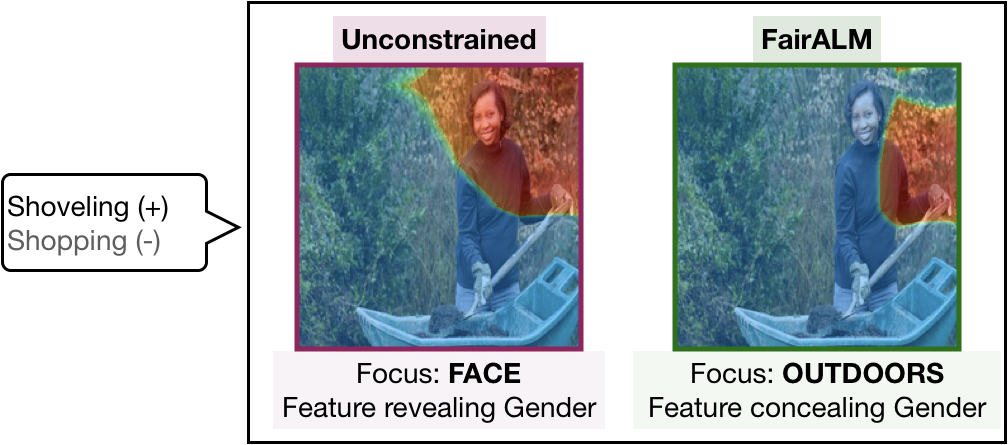}
	\end{minipage}
	\vskip 4pt
\begin{minipage}{0.9\linewidth}
	\includegraphics[width=0.5\linewidth]{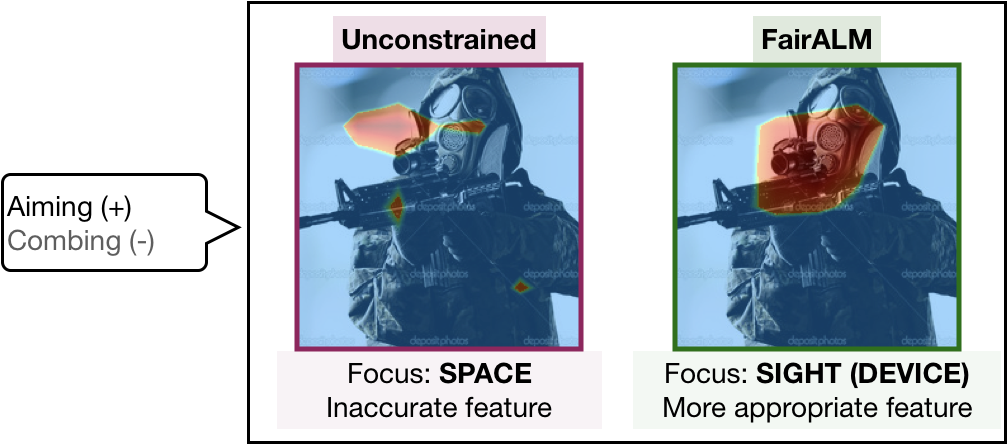} \hskip 4pt
	\includegraphics[width=0.5\linewidth]{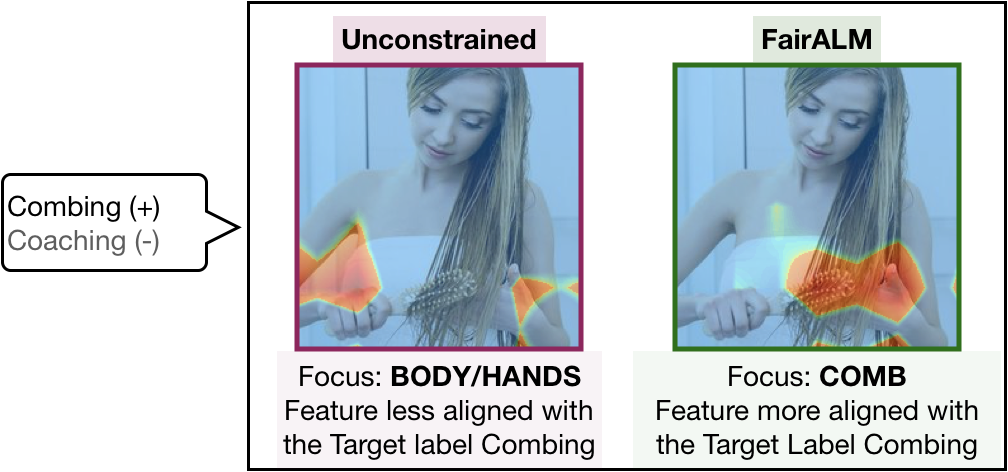}
\end{minipage}
	\vskip 4pt
\begin{minipage}{0.9\linewidth}
	\includegraphics[width=0.5\linewidth]{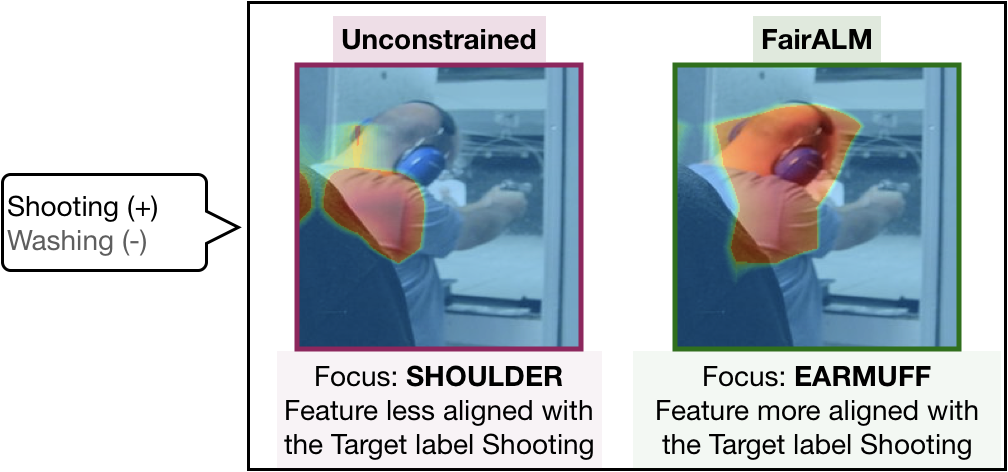} \hskip 4pt
	\includegraphics[width=0.5\linewidth]{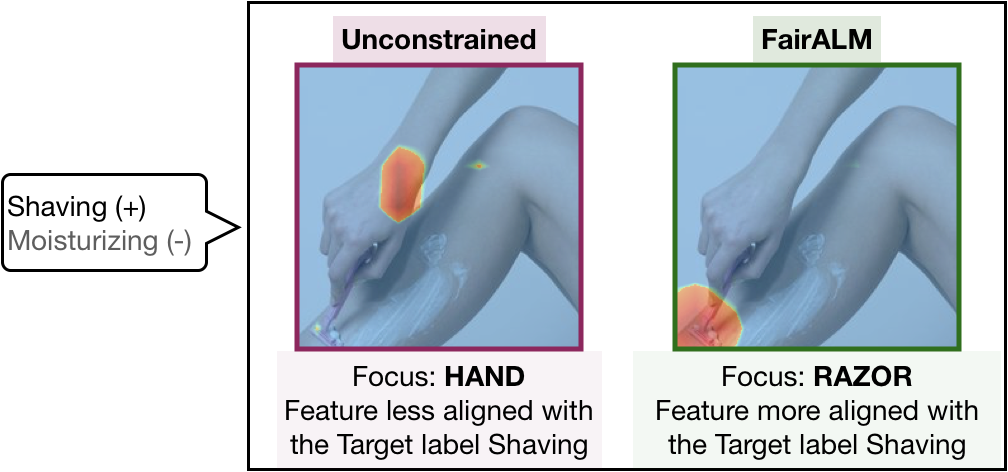}
\end{minipage}
	\vskip 4pt
\begin{minipage}{0.9\linewidth}
	\includegraphics[width=0.5\linewidth]{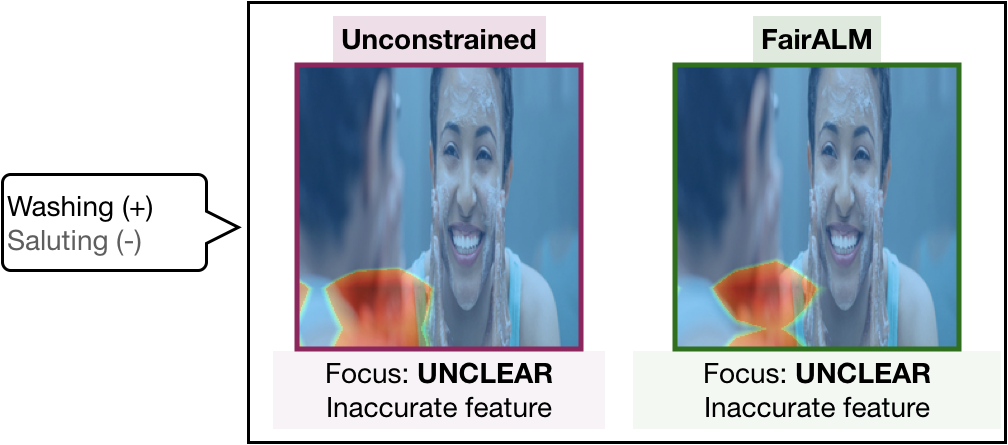} \hskip 4pt
	\includegraphics[width=0.5\linewidth]{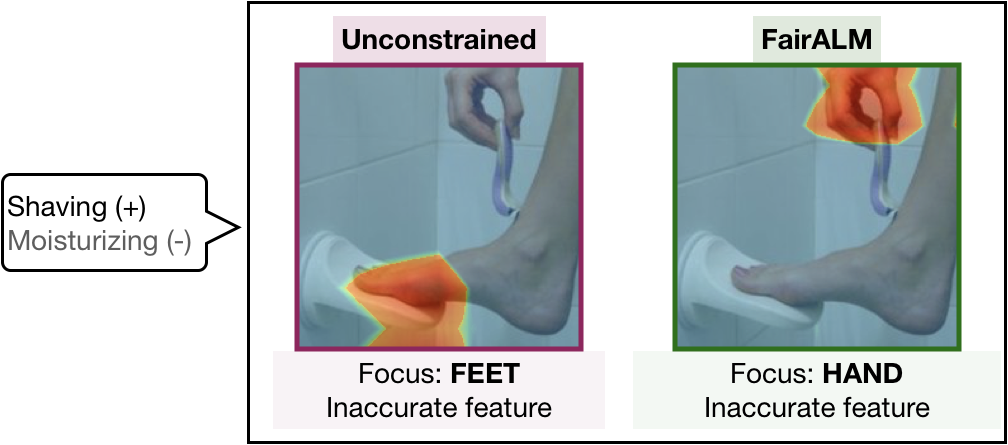}
\end{minipage}
	\caption{\label{fig:imsitu-supp-plots} \footnotesize {\textbf{Additional qualitative Results in ImSitu dataset.}} Models predict the target class $(+)$. FairALM consistently avoids gender revealing features and uses features that are more relevant to the target class. Due to the small dataset sizes, a \textit{limitation} of this experiment is shown in the last row where both FairALM and Unconstrained model look at incorrect regions. The number of such cases in FairALM is far less than those in the unconstrained model.}
\end{figure*}

\end{document}